%% file: main.tex
\ifdefined\usemsstyle

\documentclass[mnsc,blindrev]{informs3} 

\OneAndAHalfSpacedXI 

\usepackage{natbib}
 \bibpunct[, ]{(}{)}{,}{a}{}{,}%
 \def\bibfont{\small}%
 \def\bibsep{\smallskipamount}%
 \def\bibhang{24pt}%
 \def\newblock{\ }%
 \def\BIBand{and}%

\TheoremsNumberedThrough     
\ECRepeatTheorems

\EquationsNumberedThrough    

\MANUSCRIPTNO{}

\usepackage{hyperref}
\usepackage{microtype}
\usepackage{booktabs} 
\usepackage{./statistics-macros-ms}
\usepackage{pgfplotstable}
\usepackage{graphicx}
\usepackage{subcaption}
\usepackage{float}
\usepackage{soul,color}

\usepackage{algorithm}
\usepackage{algorithmic}

\usepackage{overpic}
\usepackage{tikz}
\usepackage{rotating}
\usepackage{psfrag}
\usepackage{bm}
\usepackage{placeins}
\usepackage{textcomp}
\usepackage{url}
\def\UrlBreaks{\do\/\do-}

\definecolor{innerboxcolor}{rgb}{.9,.95,1}
\definecolor{outerlinecolor}{rgb}{.6,0,.2}

\newcommand{\red}[1]{\textcolor{red}{#1}}
\newcommand{\hn}[1]{\fcolorbox{outerlinecolor}{innerboxcolor}{
    \begin{minipage}{.9\textwidth}
      \red{\bf HN Comment:} {#1}
  \end{minipage}} \\
}

\newcommand{\SDshort}[1]{
  \red{\bf [ SD: {#1} ]}
}

\newcommand{\hnshort}[1]{
  \red{\bf [ HN: {#1} ]}
}

\newcommand{\ebshort}[1]{
  \red{\bf [ EB: {#1} ]}
}

\providecommand{\comment}[1]{}
\renewcommand{\theassumption}{\Alph{assumption}}
\newcommand{\prior}{P}
\newcommand{\actionset}{\mc{A}}

\providecommand{\comment}[1]{}

\renewcommand{\theassumption}{\Alph{assumption}}

\usepackage{manyfoot}

\DeclareNewFootnote{A}
\DeclareNewFootnote{B}
\renewcommand\thefootnoteB{*}
\let\footnoteR\footnoteB
\let\footnote\footnoteA

\begin{document}



\RUNTITLE{Distilled Thompson Sampling via Imitation Learning}

\TITLE{Distilled Thompson Sampling: Practical and Efficient Thompson Sampling via Imitation Learning}

\ARTICLEAUTHORS{%
\AUTHOR{Hongseok Namkoong$^*$}
\AFF{Decision, Risk, and Operations Division, Columbia Business School, New York, NY 10027, \EMAIL{namkoong@gsb.columbia.edu}}
\AUTHOR{Samuel Daulton$^*$}
\AFF{Central Applied Science, Meta, Menlo Park, CA 94025, \EMAIL{sdaulton@meta.com}} 

\AUTHOR{Eytan Bakshy}
\AFF{Central Applied Science, Meta, Menlo Park, CA 94025, \EMAIL{ebakshy@meta.com}} 
} 

\ABSTRACT{\input{abstract}}


 \HISTORY{This paper was first submitted on May, 2021.}

\maketitle

%



\else

\ifdefined\useorstyle

\documentclass[opre,nonblindrev]{informs3}

\DoubleSpacedXI 

\usepackage{endnotes}
\let\footnote=\endnote
\let\enotesize=\normalsize
\def\notesname{Endnotes}%
\def\makeenmark{\hbox to1.275em{\theenmark.\enskip\hss}}
\def\enoteformat{\rightskip0pt\leftskip0pt\parindent=1.275em
  \leavevmode\llap{\makeenmark}}



\usepackage{natbib}
 \bibpunct[, ]{(}{)}{,}{a}{}{,}%
 \def\bibfont{\small}%
 \def\bibsep{\smallskipamount}%
 \def\bibhang{24pt}%
 \def\newblock{\ }%
 \def\BIBand{and}%

\TheoremsNumberedThrough     
\ECRepeatTheorems

\EquationsNumberedThrough    


\usepackage{hyperref}
\usepackage{microtype}
\usepackage{booktabs} 
\usepackage{./statistics-macros-ms}
\usepackage{pgfplotstable}
\usepackage{graphicx}
\usepackage{subcaption}
\usepackage{float}
\usepackage{soul,color}

\usepackage{algorithm}
\usepackage{algorithmic}

\usepackage{overpic}
\usepackage{tikz}
\usepackage{rotating}
\usepackage{psfrag}
\usepackage{bm}
\usepackage{placeins}
\usepackage{textcomp}
\usepackage{url}
\def\UrlBreaks{\do\/\do-}

\definecolor{innerboxcolor}{rgb}{.9,.95,1}
\definecolor{outerlinecolor}{rgb}{.6,0,.2}

\newcommand{\red}[1]{\textcolor{red}{#1}}
\newcommand{\hn}[1]{\fcolorbox{outerlinecolor}{innerboxcolor}{
    \begin{minipage}{.9\textwidth}
      \red{\bf HN Comment:} {#1}
  \end{minipage}} \\
}

\newcommand{\SDshort}[1]{
  \red{\bf [ SD: {#1} ]}
}

\newcommand{\hnshort}[1]{
  \red{\bf [ HN: {#1} ]}
}

\newcommand{\ebshort}[1]{
  \red{\bf [ EB: {#1} ]}
}

\providecommand{\comment}[1]{}
\renewcommand{\theassumption}{\Alph{assumption}}
\newcommand{\prior}{P}
\newcommand{\actionset}{\mc{A}}

\providecommand{\comment}[1]{}

\renewcommand{\theassumption}{\Alph{assumption}}

\usepackage{manyfoot}

\DeclareNewFootnote{A}
\DeclareNewFootnote{B}
\renewcommand\thefootnoteB{*}
\let\footnoteR\footnoteB
\let\footnote\footnoteA

\begin{document}



\RUNTITLE{Distilled Thompson Sampling via Imitation Learning}

\TITLE{Distilled Thompson Sampling: Practical and Efficient Thompson Sampling via Imitation Learning}

\ARTICLEAUTHORS{%
\AUTHOR{Hongseok Namkoong$^*$}
\AFF{Decision, Risk, and Operations Division, Columbia Business School, New York, NY 10027, \EMAIL{namkoong@gsb.columbia.edu}}
\AUTHOR{Samuel Daulton$^*$}
\AFF{Central Applied Science, Meta, Menlo Park, CA 94025, \EMAIL{sdaulton@meta.com}} 

\AUTHOR{Eytan Bakshy}
\AFF{Central Applied Science, Meta, Menlo Park, CA 94025, \EMAIL{ebakshy@meta.com}} 
} 

\ABSTRACT{\input{abstract}}


 \HISTORY{This paper was first submitted on August, 2023.}

\maketitle

%


\else


\documentclass[11pt]{article}
\usepackage[numbers]{natbib}
\usepackage{fullpage}
\usepackage{hyperref}
\usepackage{microtype}
\usepackage{booktabs} 
\usepackage{./statistics-macros}
\usepackage{pgfplotstable}
\usepackage{graphicx}
\usepackage{subcaption}
\usepackage{float}
\usepackage{soul,color}

\usepackage{algorithm}
\usepackage{algorithmic}

\usepackage{overpic}
\usepackage{tikz}
\usepackage{rotating}
\usepackage{psfrag}
\usepackage{bm}
\usepackage{placeins}
\usepackage{textcomp}
\usepackage{url}
\def\UrlBreaks{\do\/\do-}

\definecolor{innerboxcolor}{rgb}{.9,.95,1}
\definecolor{outerlinecolor}{rgb}{.6,0,.2}

\newcommand{\red}[1]{\textcolor{red}{#1}}
\newcommand{\hn}[1]{\fcolorbox{outerlinecolor}{innerboxcolor}{
    \begin{minipage}{.9\textwidth}
      \red{\bf HN Comment:} {#1}
  \end{minipage}} \\
}

\newcommand{\SDshort}[1]{
  \red{\bf [ SD: {#1} ]}
}

\newcommand{\hnshort}[1]{
  \red{\bf [ HN: {#1} ]}
}

\newcommand{\ebshort}[1]{
  \red{\bf [ EB: {#1} ]}
}

\providecommand{\comment}[1]{}
\renewcommand{\theassumption}{\Alph{assumption}}
\newcommand{\prior}{P}
\newcommand{\actionset}{\mc{A}}

\providecommand{\comment}[1]{}

\renewcommand{\theassumption}{\Alph{assumption}}

\usepackage{manyfoot}

\DeclareNewFootnote{A}
\DeclareNewFootnote{B}
\renewcommand\thefootnoteB{*}
\let\footnoteR\footnoteB
\let\footnote\footnoteA

\makeatletter
\long\def\@makecaption#1#2{
  \vskip 0.8ex
  \setbox\@tempboxa\hbox{\small {\bf #1:} #2}
  \parindent 1.5em  
  \dimen0=\hsize
  \advance\dimen0 by -3em
  \ifdim \wd\@tempboxa >\dimen0
  \hbox to \hsize{
    \parindent 0em
    \hfil
    \parbox{\dimen0}{\def\baselinestretch{0.96}\small
      {\bf #1.} #2
    }
    \hfil}
  \else \hbox to \hsize{\hfil \box\@tempboxa \hfil}
  \fi
}
\makeatother

\begin{document}
\abovedisplayskip=8pt plus0pt minus3pt
\belowdisplayskip=8pt plus0pt minus3pt

\begin{center}
  {\LARGE Distilled Thompson Sampling: Practical and Efficient \\Thompson Sampling via Imitation Learning} \\
  \vspace{.5cm} {\Large Hongseok Namkoong$^{1}$\footnoteR{Equal
      contribution.}
    ~~~ Samuel Daulton$^{*2}$
    ~~~Eytan Bakshy$^{3}$ } \\
  \vspace{.2cm}
  $^{1}$Decision, Risk, and Operations Division, Columbia Business School \\
  $^{2, 3}$Meta Central Applied Science \\
  \vspace{.2cm}
  \texttt{namkoong@gsb.columbia.edu, sdaulton@meta.com, ebakshy@meta.com}
\end{center}


\begin{abstract}
  \input{abstract}
\end{abstract}

\fi


\input{introduction-long}
\input{formulation}
\input{imitation}
\input{experiments}
\input{regret}
\input{contextual-gp}
\input{generalization}
\input{discussion}

\setlength{\bibsep}{.2em}
\bibliography{neurips/bib}
\bibliographystyle{abbrvnat}

\ifdefined\usemsstyle

\ECSwitch


\ECHead{Appendix}

\else
\ifdefined\useorstyle
\ECSwitch


\ECHead{Appendix}

\else
\newpage
\appendix
\fi
\fi

\input{wasserstein}
\input{proof-decomposition}
\input{proof-regret}
\input{proof-generalization}
\input{experiment_details}
\input{time_complexity}



\end{document}


%% file: abstract.tex
Thompson sampling (TS) has emerged as a robust technique for contextual bandit
problems. However, TS requires posterior inference and optimization for action
generation, prohibiting its use in many online platforms where latency and
ease of deployment are of concern.  We operationalize TS by proposing a novel
imitation-learning-based algorithm that distills a TS policy into an explicit
policy representation, allowing fast decision-making and easy deployment in
mobile and server-based environments. Using batched data collected under the
imitation policy, our algorithm iteratively performs offline updates to the TS
policy, and learns a new explicit policy representation to imitate it.
Empirically, our imitation policy achieves performance comparable to batch TS
while allowing more than an order of magnitude reduction in decision-time
latency. Buoyed by low latency and simplicity of implementation, our algorithm
has been successfully deployed in multiple video upload systems for Meta.  Using a
randomized controlled trial, we show our algorithm resulted in significant
improvements in video quality and watch time. 



%% file: introduction-long.tex
\section{Introduction}
\label{section:introduction}

In the past decade, Thompson sampling~\citep{Thompson33} has emerged as a
powerful algorithm for contextual bandit problems. The underlying principle is
simple: an action is chosen with probability proportional to it being optimal
under the current posterior distribution. Driven by the algorithm's strong
empirical performance~\citep{Scott10, ChapelleLi11, MayLe11}, many authors
have recently established rigorous performance
guarantees~\citep{KaufmannKoMu12, AgrawalGo13a, AgrawalGo13b,
  GopalanMaMa14,HondaTa14, RussoVa14c, AbeilleLa17}.  Thompson sampling is
increasingly being applied to a broad range of applications including revenue
management~\citep{FerreiraSiWa18}, internet
advertising~\citep{GraepelCaBoHe10, AgarwalLoTrXiZh14, SchwartzBrFa17}, and
recommendation systems~\citep{KawaleBuKvTrCh15}.

Despite its conceptual simplicity and strong performance, Thompson sampling
can be difficult to deploy in practice. Thompson sampling consists of two
steps: \emph{posterior sampling} and \emph{optimization}. \emph{Posterior
  sampling} requires evaluating a potentially large number of actions from a
well-calibrated probabilistic model.  Accurately calibrating uncertainty is
important for optimally trading off exploration and exploitation, and is
critical to practical performance~\citep{RiquelmeTuSn18}. Large-scale
probabilistic machine learning models based on deep networks show much promise
as they can adaptively learn good feature representations for uncertainty
calibration~\citep{WangYe20}.  However, sampling from these probabilistic
models can be demanding in terms of computation and memory. While approximate
inference methods with better runtime characteristics exist, they often
produce poorly calibrated uncertainty estimates that lead to poorer empirical
performance~\citep{RiquelmeTuSn18}. The second step, \emph{optimization},
solves for a reward-optimizing action under the posterior sample. This can
also be prohibitively expensive when the action space is large or
continuous. For example, an advertising platform that matches advertisers to
users at each time period has to solve combinatorial optimization problems
real-time in order to run Thompson sampling~\citep{Mas-ColellWhGr95}.

For typical online platforms, low latency---real-time computational
performance---is critical for user satisfaction and retention.  The
\emph{online} nature of the computation required for Thompson sampling thus
poses a substantive challenge to deploying it in large-scale internet
services.  These challenges are especially pronounced in resource-constrained
mobile applications, a ubiquitous modality for modern internet applications:
as of 2018, an estimated $52.2\%$ of worldwide web traffic was generated by
mobile devices~\citep{Statistica19}. Mobile applications require decisions to
be made in a fast and memory-efficient manner, and on-device decision-making
is important to good user experience in domains such as adaptive video
streaming~\citep{mao2019abr} and social media
ranking~\citep{petrescu2016client}. However, the majority of
internet-connected mobile devices have limited memory, and utilize low-end
processors that are orders of magnitude slower than server-grade
devices~\citep{Bhardwaj19,wu19}. As affordable, compute-limited mobile devices
are increasingly adopted in developing economies~\citep{Ricciardi19}, the
ability to deploy cutting-edge decision algorithms on diverse computing
infrastructure is important for democratization of technology and long term
business growth.

 
Software development cost is another core practical consideration when
implementing contextual bandit algorithms in large-scale online platforms.
Long-term software development cost is commonly referred to as tech debt,
which is incurred when a suboptimal, myopic development plan is followed in
lieu of one that requires (sometimes much) higher initial effort, but less
future work. Avoiding tech debt is critical to a reliable and scalable
service~\citep{Sculley15hidden, RamasubbuKe16, BankerLiRa21}, but contextual
bandit systems are challenging due to their high complexity: they require
temporal feedback loops consisting of different pipelines on exploration, data
logging, policy updates, and deployment~\citep{agarwal2016decision}.  The
\emph{online} nature of the complex numerical routines required by Thompson
sampling significantly exacerbate these practical
difficulties. \emph{Real-time} posterior sampling and action optimization
leads the overall system to be cumbersome and hard to debug, posing challenges
to reliable software development.

\vspace{10pt}
\begin{example}[Video Transcoding]
\label{example:video_transcoding}
As our main real-world application, we study video uploads for large online
platforms. Video is an increasingly popular medium on social networks, but
uploading video is still a technically challenging problem, where limited
bandwidth and compute capacity---particularly problematic on mobile
devices---leads to unsuccessful uploads. When a user requests a video be
uploaded to a social media service, the service must choose the desired video
quality (bitrate) for transcoding the video before uploading.  Video needs to
be optimally transcoded considering quality, and success of file upload. It is
preferable to upload videos at a high quality because it can lead to a better
viewer experience (if the viewer has a sufficiently good network
connection). However, higher quality videos have larger file sizes, making it
more likely to fail to upload: larger files take longer time to upload,
increasing the likelihood that the network connection to fail, or the user to
grow frustrated and cancel the upload.

We are interested in an online platform who wish to make contextual decisions
about how to optimally transcode a video at upload time.  Making such
decisions quickly is critical for user satisfaction; low latency is
particularly important for popular short-form videos uploaded on Tik-Tok,
Snapchat, and Instagram, where videos are captured and uploaded frequently and
in real-time. Although transcoding decisions needs to be made quickly in order
to be responsive and keep the user engaged, most upload requests come from
resource-constrained mobile devices. 
\ifdefined\usemsstyle $\diamond$ \fi
\end{example}
\ifdefined\usemsstyle
\else
\vspace{-5pt}
\fi

\begin{figure}
    \centering
    \includegraphics[width=.8\textwidth]{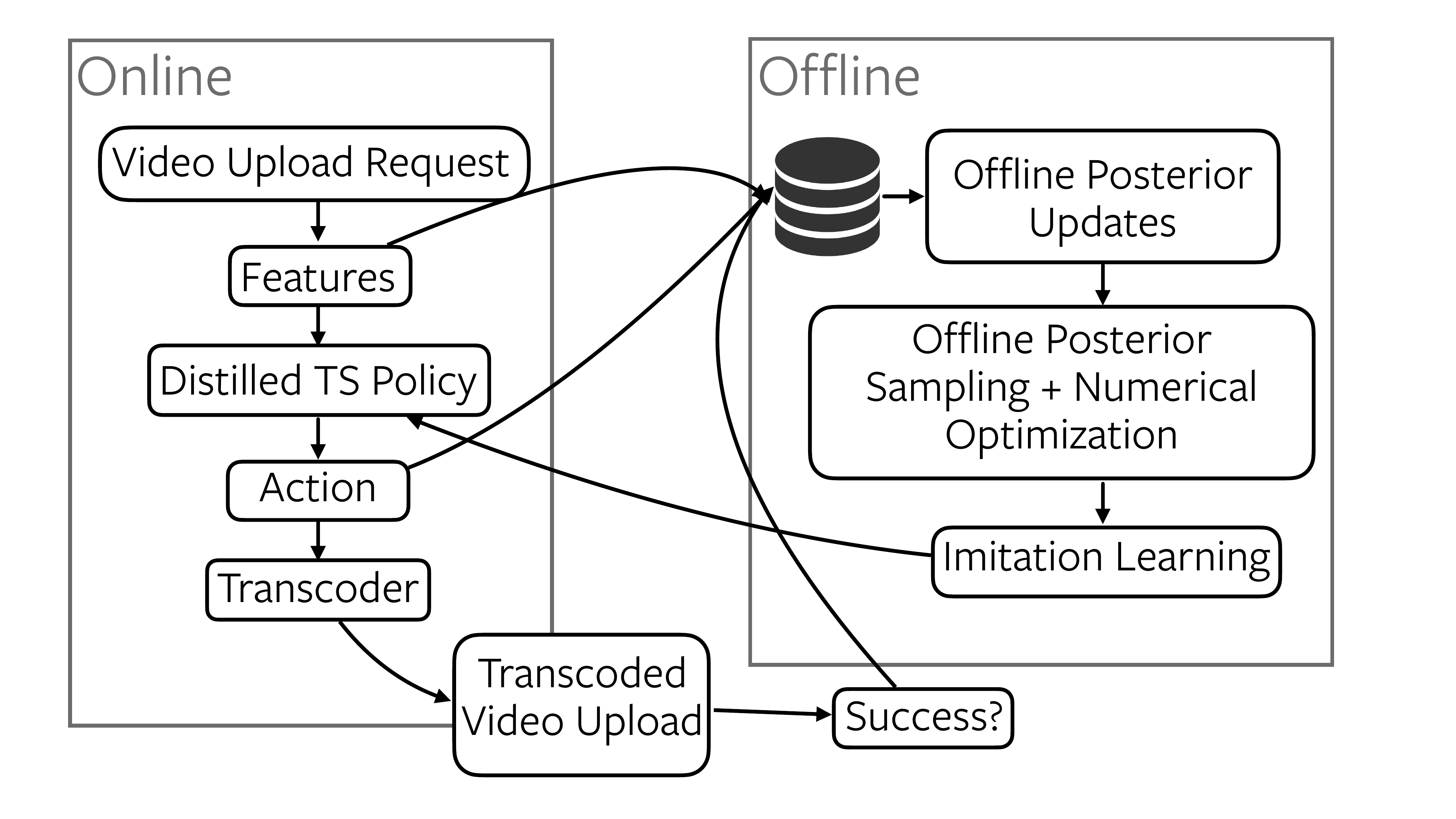}
    \caption{\label{fig:distilled_ts_videos_diagram} An illustration of
      distilled TS on the example video uploads application, described in
      Example \ref{example:video_transcoding}. \emph{Online} action generation
      is performed asynchronously on resource-constrained mobile edge devices
      whereas batched policy updates are performed \emph{offline} on
      powerful backend servers.}
\end{figure}

The problem motivation goes beyond our main application. There are numerous
examples of decision-making problems on online platforms where latency and
system complexity are of central concern.

\begin{example}[Advertising on third party systems] Every time a user arrives
  to a third party webpage (e.g. New York Times), the advertising platform
  (e.g. Google Ads) decides which ad to show in order to maximize conversion.
  Latency is important to good user experience~\citep{agarwal2016decision},
  and curbing system complexity increases service
  reliability~\citep{Sculley15hidden}.  \ifdefined\usemsstyle $\diamond$ \fi
\end{example}
\ifdefined\usemsstyle
\else

\vspace{-10pt}
\fi
\begin{example}[Ranking]
  When a user logs in, an internet service chooses a list of items to display
  to the user in order to maximize revenue or engagement.  Concrete examples
  include ranking news articles (Microsoft Network, MSN), products (online
  marketplaces like Amazon and Airbnb), and content (Facebook and LinkedIn
  feed). In all of these cases, latency is central to user satisfaction, but
  mobile edge devices and front-end servers are resource
  constrained~\citep{agarwal2016decision}. For instance, there has been work from Meta Facebook on performing
  secondary ranking on device to avoid server communication latency and to only
  display content that has been downloaded
  completely~\citep{petrescu2016client}.
  \ifdefined\usemsstyle $\diamond$ \fi
\end{example}
\ifdefined\usemsstyle
\else
\vspace{-10pt}
\fi

\begin{example}[Personalized Pricing]
  As a customer enters a virtual platform, the system generates a personalized
  price based on market conditions and user-specific contexts. Electronic
  commerce firms and airlines use price controls to manage
  revenue~\citep{TalluriVa04, DenBoer15}, and two-sided online marketplaces
  (e.g. Uber, Lyft, Airbnb) dynamically set prices on both sides of the market
  to reduce supply-demand imbalance. In both cases, latency is important for a
  satisfactory user experience.  \ifdefined\usemsstyle $\diamond$ \fi
\end{example}
\ifdefined\usemsstyle
\else
\vspace{-10pt}
\fi

\paragraph{Methodology} 
Motivated by aforementioned challenges in implementing and deploying Thompson
sampling on online platforms, we develop and analyze a method that maintains
an explicit policy representation designed to imitate Thompson sampling.  In
order to avoid computationally demanding routines \emph{online}, our algorithm
simulates and imitates a Thompson sampling policy \emph{offline}.  An explicit
policy representation can efficiently generate actions real-time even in large
action spaces, without requiring real-time posterior inference or numerical
optimization. An illustration of how this methodology can be applied to video
transcoding (Example \ref{example:video_transcoding}) is provided in Figure
\ref{fig:distilled_ts_videos_diagram}. This allows leveraging state-of-the-art
Bayesian models---such as Gaussian processes parameterized by deep neural
networks---and optimization solvers \emph{offline}, while maintaining low
latency on resource-constrained computing modalities such as low-end mobile
devices.\footnote{More generally, optimization can be a challenge for
  non-Bayesian methods. Although outside of the scope of this paper,
  generalizing our imitation framework to other policies will likely yield
  fruit in separating optimization from online action-generation.}  During
operation, actions can be generated efficiently from the distilled policy by
sampling from a parameterized distribution, allowing fast and asynchronous
interaction with users.  For example, recent engineering progress allows
generating actions using an industrial-scale neural network model in 0.3880
milliseconds~\citep{ColemanEtAl17}.

By performing posterior updates and mimicking the behavior of Thompson
sampling offline, we are able to move complex numerical routines from
resource-constrained mobile devices to backend servers, and reduce long-term
software development costs (tech debt).  Such offline procedures using batched
observations can be easily implemented using modern industry machine learning
pipelines~\citep{gauci2018horizon,fujimoto2019ope}. This allows leveraging the
recent remarkable progress in machine learning software infrastructure, such
as engineering best practices and tools for reliable testing \&
deployment\footnote{As an example, a dedicated top peer-reviewed conference
  for ML systems~\url{https://mlsys.org/} was recently established, and is
  undergoing rapid growth at the forefront of academia and industry. This
  community focuses on improving the efficiency of ML systems from an
  \emph{operational} perspective.}.



\paragraph{Practical impact} 

Empirically, we evaluate our imitation algorithm on several benchmark problems
and a real-world dataset for selecting optimal video transcoding
configurations (Section \ref{section:experiments}).  In all of our
experiments, our imitation algorithm performs as well as batch on-policy
Thompson sampling in terms of cumulative regret, while reducing decision-time
latency by an order of magnitude.  Buoyed by low latency and simplicity of
implementation showcased in our empirical benchmarking efforts, our imitation
learning policies have been used in video upload systems across Meta products,
which are leading social networking services. Our contextual policy tunes the
bitrates for video uploads based on contextual features such as download
bandwidth, device model, operating system, connection class (2G, 3G, 4G),
country, and video features which include source resolution, bitrate, and file
size. 

To assess the impact of our algorithm, we ran internal randomized controlled
trials (RCT) on each of the aforementioned products. We find our algorithm
achieves significant improvements in video quality, which we measure using the
fraction of videos with quality preserved at 1080p (high resolution). 
Our RCTs show up to 5x improvements over existing video upload policies on all
surfaces.

The RCTs show significant increases in topline metrics that are of importance
at the company level.  Due to better video quality, we observed increased
video watch times on multiple products: $1.1\%$ on Facebook iOS Feed videos,
$0.77\%$ on Facebook Android Feed videos, $0.27\%$ on Facebook Android
Stories, $0.45\%$ on Instagram Stories. In addition, our contextual policies
boosted interaction metrics on several products: increases in meaningful
social interactions of $0.15\%$ and $0.14\%$ on Facebook Android Stories and
Facebook Android Feed, respectively, and an increase in interactions of
$0.26\%$ on Instagram Stories. All findings were significant at the $95\%$
level.

Buoyed by these results, our contextual policy has been deployed across
multiple product verticals including Facebook Feed, Stories, Reels and
Instagram Stories and Reels.  Our algorithm has been independently applied to
both iOS and Android apps for all aforementioned products, and is reliably
handling millions of uploads each day.

\paragraph{Theoretical contributions} 
To understand the strong practical advantages we showcase, we take initial
steps toward a principled understanding of our imitation algorithm.  Since our
(batch) updates to the Thompson sampling policy are based on observations
generated by the imitation policy, our algorithm emulates an \emph{off-policy}
version of Thompson sampling which may diverge from its on-policy
counterpart. Due to its off-policy nature, an uninformed and pessimistic view
of our procedure states that any initially small deviation between the
imitation policy and Thompson sampling may cascade across time. Our main
theoretical results
(Section~\ref{section:generalization}-\ref{section:imitation-controls-regret})
preclude such possibility and ensure small deviations between the imitation
policy and Thompson sampling do not magnify over time. Specifically, we show
that our imitation policy enjoys Bayes regret similar to that of batch
\emph{on-policy} Thompson sampling, up to the sum of single-step imitation
errors. We substantiate our performance guarantees in general modeling
scenarios involving contextual Gaussian processes, where a cleverly
initialized version of our algorithm (albeit impractical) achieves
advantageous Bayes regret (Section~\ref{section:regret-bound}).

Solving the imitation problem, or equivalently, finding the policy
parameterization closest to Thompson sampling, only requires unsupervised
contexts---those without corresponding actions or rewards.  On large-scale
online platforms, unsupervised contexts are typically cheap and abundant,
e.g., the entire user database provides a wealth of such contexts. In
Section~\ref{section:generalization}, we prove that each single-period
imitation error term can be controlled---with a sufficiently rich imitation
model---at the rate $O_p(1/\sqrt{N})$, where $N$ is the number of supervised
and unsupervised contexts. Combining this with our aforementioned regret bound
in Section~\ref{section:imitation-controls-regret}, our imitation algorithm
achieves Bayes regret comparable to batch on-policy Thompson sampling up to
$O(T \sqrt{\log T} / \sqrt{N})$-error, where $T$ is the number of batched
policy updates.

Despite the seemingly linear gap in Bayes regret, $N$ is typically orders of
magnitude larger than $T$ in internet applications where we can utilize the
database of users / entities. Typically, $N$ is in the order of hundreds of
millions; as of 2020, Facebook had 2.7 billion monthly active users; in our
motivating video transcoding application, the service receives millions of
video upload requests \emph{every day}, providing an effectively unlimited
number of unsupervised contexts. In contrast, the number of model updates
(horizon $T$) is relatively small, in hundreds, due to complexities of policy
deployment and nonstationary user behavior. In such practical problem
instances, our imitation policy thus enjoys Bayes regret bounds comparable to
that of batch on-policy Thompson sampling.

\section{Related work}

There is a substantial body of work on Thompson sampling and its variants that
use computationally efficient subroutines. We give a necessarily abridged
overview of how our algorithm situates with respect to the extensive
literature on bandits, approximate inference, and imitation learning.

A number of authors have showed that Thompson sampling achieves optimal regret
for multi-armed bandits~\citep{AgrawalGo12, AgrawalGo13a, KaufmannKoMu12,
  HondaTa14}. We refer the reader to the recent tutorial
by~\citet{RussoVaKaOsWe18} and references therein for a comprehensive
overview.~\citet{AgrawalGo13b, AbeilleLa17} showed regret bounds for linear
stochastic contextual bandits for a Thompson sampling algorithm with an
uninformative Gaussian prior, and~\citet{GopalanMaMa14} studied finite
parameter spaces.~\citet{RussoVa14c} established Bayesian regret bounds for
Thompson sampling with varying action sets (which includes, in particular,
contextual bandits); ~\citet{RussoVa16} provides an information-theoretic
analysis that makes explicit the dependence on the prior (see
also~\citet{BubeckEl16}). We build on the insights of~\citet{RussoVa14c}, and
show that our imitation algorithm retains the advantageous properties of batch
Thompson sampling, achieving (gap-independent) Bayes regret comparable to the
\emph{best} batch UCB algorithm.


Practical performance of Thompson sampling depends on having access to
well-calibrated probabilistic predictions.  Obtaining a balance between
predictive accuracy, computational time, and memory requirements can be
challenging in the context of large datasets with overparameterized models.
Exact posterior sampling from even the simplest Gaussian linear models has a
time complexity of $O(n^2)$, where $n$ is the number of model
parameters\footnote{This assumes the root decomposition of the covariance
  matrix has been cached, which incurs a cost of $O(n^3)$.}.  A common
strategy used by some variational inference methods is to use a mean-field
approach where parameters are assumed to be independent
\citep{blundell2015weight}. This assumption can decrease sampling costs from
$O(n^2)$ to $O(n)$, where $n$ is the number of parameters. However,
\citet{RiquelmeTuSn18} found that batch Thompson sampling using such
approaches often leads to poor empirical performance.

When exact posterior inference is not possible, approximate inference methods can be
used for posterior sampling. We refer the reader to Chapter 5
of~\citet{RussoVaKaOsWe18}'s recent tutorial for a discussion of approximation
methods in relation to Thompson sampling. Bootstrapping~\citep{eckles2014bts,
  OsbandBlPrVa16, lu2017ensemble} is a simple heuristic procedure that
maintains multiple models to approximate samples from the posterior
distribution, although maintaining multiple models is often computationally
expensive. MCMC-based methods for approximate inference, and Hamilton Monte
Carlo (HMC)~\citep{neal2011hmc} in particular, are largely regarded as the
``gold standard'' for approximate Bayesian inference.  HMC, and other
MCMC-like approaches (e.g., \citet{pmlr-v32-cheni14,welling2011sgld}) generate
an arbitrary number of posterior samples for all parameters.  While such
algorithms permit rapid evaluation of posterior samples (since the parameters
are already sampled), they require substantial memory to store multiple
samples of the parameters.  Recent methods have also considered decomposing the
covariance or precision matrix into a diagonal and low-rank component
\citep{zhang2018noisy,maddox2019swag}.  While this reduces computational
complexity and memory costs relative to using the full covariance, sampling
still incurs a time complexity of $O((n+1)\rho)$ where $\rho$ is the rank of
the covariance (or precision matrix) and $\rho$ copies of the weights must be
stored.

By pre-computing and distilling Thompson sampling, our imitation learning
framework allows the use of the most appropriate inferential procedure for the
task at hand, rather than what is feasible to run in an online setting. In
particular, the separation of online decision-making and offline computation
allows the use of state-of-the-art Bayesian methods, such as those utilizing deep neural
networks~\citep{WangYe20}. While we restrict discussion to Thompson sampling in
this work, the basic idea of offline imitation learning can be used to learn a
explicit policy representation of any complicated policy and allow
operationalization at scale.


Imitation learning methods have received much attention recently, owing to
their ability to learn complicated policies from expert
demonstrations~\citep{AbbeelNg04, RossBa10, HoEr16}.  Our approach of
minimizing the discrepancy between a parameterized policy and Thompson
sampling can be viewed as an implementation of behavioral
cloning~\citep{RossBa10, SyedSc10, RossGoBa11}.  Our imitation learning
procedure resembles the ``Bayesian dark knowledge'' approach from
\citet{Korattikara15}, which uses a neural network to approximate Bayesian
posterior distributions.  While most works in the imitation learning
literature study reinforcement learning problems, we focus on the more limited
contextual bandit setting, which allows us to show strong theoretical
guarantees. We anticipate the growing list of works on imitation learning to
be important in generalizing our imitation framework to the reinforcement
learning (RL) setting. To account for time dependencies in state evolutions,
both inverse RL approaches that directly model the reward~\citep{AbbeelNg04,
  SyedSc08}, and the recent advances in generative adversarial imitation
learning techniques~\citep{HoEr16, LiSoEr17} show promise in generalizing our
imitation algorithm (behavioral cloning) to RL problems.




%% file: formulation.tex
\section{Distilled Thompson sampling}
\label{section:imitation}


Reflecting typical operational scenarios on online platforms, we consider a
\emph{batch} (Bayesian) contextual bandit problem. The agent / decision-maker
generates actions \emph{real-time} as user requests come in asynchronously, and
performs batched, infrequent updates to the policy. In what follows, we
formally introduce an imitation algorithm that makes it trivial to parallelize
action generation over multiple computing nodes, even on each user's mobile
device.

Let $\Theta$ be the parameter space, and let $\theta \sim \prior$ be a prior
distribution on $\Theta$. At each time $t$, the agent observes a context,
takes an action, and receives a reward: we denote the context
$S_t \overset{\text{iid}}{\sim} \P_S$, action $A_t \in \actionset$, and reward
$R_t \in \R$. We consider a well-specified reward model class
$\{f_{\theta}: \actionset \times \mc{S} \to \R \mid \theta \in \Theta\}$
\begin{equation*}
  f_{\theta}(a, s) = \E[R_t \mid \theta, A_t = a, S_t = s]~~\mbox{for all}~~a
  \in \actionset, s \in \mc{S}.
\end{equation*}
Let $H_t = (S_1, A_1, R_1, \ldots, S_{t-1}, A_{t-1}, R_{t-1})$ be the history
of observations until time $t$. Assume that regardless of $H_{t'}$ for
$t'\le t$, the mean reward at time $t$ is determined only by the
context-action pair
\begin{align*}
  \E[R_t \mid \theta, H_{t'}, S_t = s, A_t = s] = f_{\theta}(a, s),
\end{align*}
or equivalently, $R_t = f_{\theta}(A_t, S_t) + \epsilon_t$ where $\epsilon_t$
is a mean zero i.i.d. noise.

At time $t$, we denote by $\prev(t)$ the period before which the most recent
policy update occurred. For example, for a fixed batch size $\batch$
\begin{align}
  \label{eqn:batch-size}
  \prev(t) = \begin{cases}
    1 & \mbox{if}~t = 1, \ldots \batch, \\
    \batch+1 & \mbox{if}~t = \batch+1, \ldots, 2\batch, \\
    2\batch+1 & \mbox{if}~t = 2\batch+1, \ldots, 3\batch, \\
    \vdots &
   \end{cases}.
\end{align}
More generally, we allow time-varying batch sizes that are a priori unknown to
the decision maker. We use $\pi_{\prev(t)}$ to denote the policy used at time
$t$ that generates action $A_t$ based on the history $H_{\prev(t)}$ available
at the previous model update $\prev(t)$: conditional on the history
$H_{\prev(t)}$, we have $A_t \mid S_t \sim \pi_{\prev(t)}(\cdot \mid S_t)$,
where we abuse notation to suppress the dependence of $\pi_{\prev(t)}$ on the
history $H_{\prev(t)}$. In the sequential (non-batch) setting, we simply have
$\prev(t) = t$.

The agent's objective is to maximize the cumulative sum of rewards by updating
the policy $\pi_{\prev(t)}$ based on batches of context-action-reward
observations. The \emph{regret} of the agent compares the agent's cumulative
reward to the reward under the optimal action: for any fixed parameter value
$\theta \in \Theta$, the (frequentist) regret for the set of policies
$\{\pi_{\prev(t)}\}_{t \in \N}$ is
\begin{equation*}
  \regret{T}{\{\pi_{\prev(t)}\}_{t \in \N}}{\theta} 
  \defeq
  \sum_{t=1}^T \E\left[ \max_{a \in \actionset} f_\theta(a, S_t) -
    f_\theta(A_t, S_t) \mid \theta \right].
\end{equation*}
For simplicity, we assume $\argmax_{a \in \actionset} f_{\theta}(a, s)$ is
nonempty almost surely.
We assume the agent's prior, $\prior$, is \emph{well-specified}\footnote{ When
  the prior is misspecified so that the Thompson sampling policy uses $Q$
  instead of $P$, we have the equivalence as noted by~\citet{RussoVa14c}
\begin{align*}
  \E_{\theta \sim \prior} [\regret{T}{\{\pi_{\prev(t)}\}_{t \in \N}}{\theta}]
  \le \linfstatnorm{\frac{dP}{dQ}} \E_{\theta \sim Q} [\regret{T}{\{\pi_{\prev(t)}\}_{t \in \N}}{\theta}],
\end{align*}
where $dP/dQ$ is the Radon-Nikodym derivative of $P$ with respect to
$Q$. While misspecified priors can incur substantially higher
regret~\citep{LiuLi16} in the worst-case, empirical evidence suggests Thompson
sampling is a strong algorithm in practice~\citep{Scott10, Granmo10,
  ChapelleLi11, MayLe11, FerreiraSiWa18, GraepelCaBoHe10, AgarwalLoTrXiZh14,
  KawaleBuKvTrCh15, SchwartzBrFa17, agarwal2016decision}. }, a key
(standard) assumption that drives our subsequent analysis. Under the prior
$\prior$ over $\theta \in \Theta$, the Bayes regret is simply the frequentist
regret averaged over $\theta \sim \prior$
\begin{equation*}
  \bayesregret{T}{\{\pi_{\prev(t)}\}_{t \in \N}}
  \defeq \E_{\theta \sim \prior} [\regret{T}{\{\pi_{\prev(t)}\}_{t \in \N}}{\theta}] 
  = \sum_{t=1}^T \E_{\theta \sim \prior}\left[ \max_{a \in \actionset} f_\theta(a, S_t) - f_\theta(A_t, S_t) \right].
\end{equation*}

Based on the history $H_{\prev(t)}$, batch Thompson sampling plays an action
according to the posterior probability of the action being optimal. The
posterior probabilities are computed based on the prior $\prior$ and
previously observed context-action-reward tuples. At time $t$, this is often
implemented by 
\begin{equation*}
  \mbox{sampling from the posterior}~~
  \theta_t \sim \prior(\theta \in \cdot \mid H_{\prev(t)}, S_t)~~\mbox{and solving}~~
  \bar{A}_t \in \argmax_{a \in \actionset} f_{\theta_t}(a, S_t).
\end{equation*}
By definition, Thompson sampling enjoys the optimality property
$\bar{A}_t \mid H_{\prev(t)}, S_t \eqd A_t\opt \mid H_{\prev(t)}, S_t$ where
$A_t\opt \in \argmax_{a \in \actionset} f_{\theta}(a, S_t)$ and $\theta$ is
the true parameter drawn from the prior $\prior$. Throughout, we assume
$\bar{A}_t \mid H_{\prev(t)}, S_t$ is independent of all else.


%% file: imitation.tex
To address challenges in implementing Thompson sampling real-time, we develop
an imitation learning algorithm that separates \emph{online} action generation
from computationally intensive steps like posterior sampling and
optimization. Our algorithm maintains an explicit policy representation that
emulates the batch (off-policy) Thompson sampling policy by simulating its
actions \emph{offline}. At decision time, the algorithm generates an action
simply by sampling from the current policy representation, which is
straightforward to implement and computationally efficient to run real-time.
We summarize an idealized form of our method in Algorithm~\ref{alg:imitation},
where conditional on the history $H_{\prev(t)}$ generated by the imitation
policy
\begin{equation}
  \label{eqn:bots}
  \bar{\pi}_{\prev(t)}( a\mid s)~~\mbox{is the \emph{batch off-policy Thompson sampling policy at time}}~t.
\end{equation}
This policy is different from the true, batch on-policy Thompson sampling
since the imitation policy generates actions based on which rewards are
observed. Nevertheless, we will show that our algorithm enjoys Bayes regret
comparable to batch on-policy Thompson sampling.

At each time $t$, our algorithm observes a context $S_t$, and plays an action
drawn from its explicit policy representation. Formally, we parameterize our
policy $\pi^m(a \mid s)$ with a model class $m \in \mc{M}$. For example,
$\mc{M}$ can be a neural network that takes as input a context and outputs a
distribution over actions. We generate actions by sampling from the current
policy $A_t \sim \pi_{\prev(t)}^m(\cdot \mid S_t)$, which can be easily
implemented to run with low latency on resource-constrained computing
infrastructure such as mobile devices.  The agent uses a batch of
context-action-reward tuples to update its posterior on the parameter
$\theta \in \Theta$ \emph{offline}.  Although this step requires posterior
inference that may be too burdensome to run real-time, our method allows
running it offline on a different computing node, so that it does not affect
latency. Using the updated posterior
$\theta_t \sim \P(\cdot \mid H_{\prev(t)})$, the agent then simulates actions
drawn by the Thompson sampling policy by computing the maximizer
$\bar{A}_t(s) \in \argmax_{a \in \mc{A}} f_{\theta_t}(a, s)$, for a range of
values $s \in \mc{S}$. Using these simulated context-action pairs, we learn an
explicit policy representation that \emph{imitates} the observed actions of
the Thompson sampling policy.

\begin{algorithm}[H]
  \caption{Imitating Batch Thompson Sampling}
  \label{alg:imitation}
  \begin{algorithmic}[1]
    \STATE Input: prior $\prior$ on parameter space $\Theta$, reward model
    class $\{f_{\theta}(\cdot, \cdot)\}$, imitation policy model class
    $\{\pi^m: m \in \mc{M}\}$, notion of distance $D$ for probabilities
    \STATE Initialize
    $m \gets \argmin_{m \in \mc{M}} \E_{S \sim \P_S}
    [\disc{\bar{\pi}_0}{\pi^{m}}{S}]$
    \FOR{$t=1$ \textbf{to} $T$}
    \STATE Observe $S_t$, sample $A_t \sim \pi^{m}_{\prev(t)}(\cdot \mid S_t)$, receive $R_t$
    \IF{$t+1 = \prev(t+1)$} \STATE Update model
    $m \gets \argmin_{m \in \mc{M}} \E_{S \sim \P_{S}}
    [\disc{\bar{\pi}_{\prev(t+1)}}{\pi^{m}}{S}]$ \emph{offline} \ENDIF
    \ENDFOR
  \end{algorithmic}
\end{algorithm}

Dropping the time subscript to simplify notation, the imitation learning
problem
\begin{equation}
  \label{eqn:imitation}
  \minimize_{m \in \mc{M}} \E_{S \sim \P_S}
  \left[\disc{\bar{\pi}}{\pi^m}{S}\right].
\end{equation}
learns a model $m \in \mc{M}$ minimizing a measure of discrepancy
$\disc{\cdot}{\cdot}{S}$ between the two distributions on $\actionset$,
conditional on the context $S$. As the imitation
objective~\eqref{eqn:imitation} cannot be computed analytically, we provide
efficient approximation algorithms. To instantiate
Algorithm~\ref{alg:imitation}, we fix Kullback-Leibler (KL) divergence as the
notion of discrepancy between probabilities and present finite-sample
approximations based on observed contexts and simulated actions from the
off-policy Thompson sampling policy $\bar{\pi}_t$.
For probabilities $q^1$ and $q^2$ on $\actionset$ such that
$q^1, q^2 \ll \nu$ for some $\sigma$-finite measure $\nu$ on $\actionset$,
the KL divergence between $q^1$ and $q^2$ is
  $\dkl{q^1}{q^2} \defeq
  \int_{\actionset} \log \frac{dq^1/d\nu}{dq^2/d\nu}(a) d\nu(a)$,
where we use $\frac{dq^1}{d\nu}$ and $\frac{dq^2}{d\nu}$ to denote Radon-Nikodym
derivatives of $q^1$ and $q^2$ with respect to $\nu$. 
For two policies $\pi^1$ and $\pi^2$, we define
\begin{equation*}
  \dklpolicy{\pi^1}{\pi^2}{S} \defeq \dkl{\pi^1(\cdot \mid S)}{\pi^2(\cdot \mid S)},
\end{equation*}
where we use $\pi^1, \pi^2$ to also denote their conditional densities over
$\actionset$.

The imitation problem~\eqref{eqn:imitation} with
$\disc{\cdot}{\cdot}{S} = \dklpolicy{\cdot}{\cdot}{S}$ is equivalent to
maximizing log likelihood
\begin{equation}
  \label{eqn:mle}
  \maximize_{m \in \mc{M}}
  \E_{S \sim \P_S, \bar{A} \sim \bar{\pi}(\cdot \mid S)}[\log \pi^m(\bar{A} \mid S)].
\end{equation}
In the following, we write
$\E[\cdot] = \E_{S \sim \P_S, \bar{A} \sim \bar{\pi}(\cdot \mid S)}[\cdot]$
for simplicity. In the maximum likelihood estimation (MLE)
problem~\eqref{eqn:mle}, the data comprises of context-action pairs. First,
contexts are generated under the marginal distribution $S \sim \P_S$
independent of everything else. Conditional on the context, actions are
simulated from the batch off-policy Thompson sampling policy
$\bar{A} \sim \bar{\pi}(\cdot \mid S)$. The MLE problem~\eqref{eqn:mle} finds
a model $m \in \mc{M}$ maximizing the likelihood of observing actions
generated by $\bar{\pi}_{\prev(t)}$.

The imitation objective $m \mapsto \E[\log \pi^m(\bar{A} \mid S)]$ involves an
expectation over the unknown marginal distribution of contexts $\P_S$ and
actions generated by the Thompson sampling policy $\bar{\pi}(\cdot \mid
S)$. Although the expectation over $S \sim \P_S$ involves a potentially
high-dimensional integral over an unknown distribution, sampling from this
distribution is usually very cheap since the observations $S \sim \P_{S}$ can
be ``unsupervised'' in the sense that no corresponding action/reward are
necessary. For example, it is common for online platforms to maintain a
database of features $S$ for all of its users. Using these contexts, we can
solve the MLE problem~\eqref{eqn:mle} efficiently via stochastic gradient
descent methods~\citep{KushnerYi03, Duchi18}. In
Section~\ref{section:generalization}, we show that it is easy to solve the
imitation problem~\eqref{eqn:imitation} to high accuracy by using cheap
unsupervised contexts. In Section~\ref{section:imitation-controls-regret}, we
show that our imitation algorithm enjoys Bayes regret comparable to that of
the batch on-policy Thompson sampling algorithm, up to the sum of single step
imitation errors.

For continuous action spaces with a notion of geometry, it is sometimes
natural to allow imitation policies to have slightly different support than
the Thompson sampling policy. In this scenario, we can instantiate the
abstract form of Algorithm~\ref{alg:imitation} with Wasserstein distances as
our notion of discrepancy $\disc{\cdot}{\cdot}{s}$. The subsequent theoretical
development for KL divergences has its analogue for Wasserstein distances,
which we outline in Appendix~\ref{section:wass}


%% file: experiments.tex
\section{Empirical evaluation}
\label{section:experiments}

We study the performance of our imitation learning algorithm in terms of
cumulative regret / reward and decision-time latency in a number of
datasets. Our imitation learning algorithm achieves a significant reduction in
latency on all problems and enjoys regret comparable to that of batch
on-policy Thompson sampling, avoiding compounding of imitation error over
time. Our experiments include a real-world video upload transcoding
application for an internet service receiving millions of video upload
requests per day.

\ifdefined\usemsstyle We include our open-sourced
implementation as supplementary materials.  \fi



\paragraph{Datasets} We compare our imitation algorithm alongside an array of benchmark methods on
four problem scenarios. For our first experiment, we study the \textbf{wheel
  bandit problem}, a synthetic problem constructed to require
significant exploration~\citep{RiquelmeTuSn18}. In this two-dimensional
problem, there are 5 actions and rarely seen contexts yield high rewards under
one context-dependent action. We sample $10,000$ contexts for each trial.
Specifically, two-dimensional contexts are sampled in the unit sphere with
uniform probability. The first action always has a
mean reward of $\E[r(\bm s, a_1)] = 1.2$ independent of the context, and the
mean rewards of the other actions depend on the context. If
$||\bm s||_2 \leq \delta$, then the remaining four actions are non-optimal with a
mean reward of 1. If $||\bm s||_2 > \delta$, then one of the remaining actions
is optimal---and determined by the sign of the two dimensions of $\bm s$
---with a mean reward of 50. The remaining three actions all have a mean reward of
1. All rewards are observed with zero-mean additive Gaussian noise with
standard deviation $\sigma = 0.01$. We set $\delta = 0.95$, which means the
probability of sampling a context on the perimeter
($||\bm s||_2 \geq \delta$) where one action yields a large reward is
$1-(0.95)^2 = 0.0975 \approx 10\%$.

For our second problem, we design a contextual bandit problem from a
supervised classification task. The \textbf{Mushroom UCI Dataset} \citep{misc_mushroom_73} contains
8,124 examples with 22 categorical features about the mushroom and labels
indicating if the mushroom is poisonous or not. At each time step, the forager
decides whether to eat the mushroom or not and receives a small positive
reward for eating a safe mushroom, and a large negative reward for eating an
unsafe mushroom. With equal probability, eating a poisonous mushroom lead to
illness ($r=-35$) or it may not harm the consumer ($r=5$), while a
nonpoisonous mushroom always yields a positive reward ($r=5$). The reward for
abstaining is always 0. We sample $50,000$ contexts for each trial.

Next, we turn our attention to a more realistic healthcare scenario,
\textbf{pharamacological dosage optimization}, where we wish to learn a good
dosing policy for Warfarin.  Warfarin is one of the most common anticoagulants
(blood thinner), often prescribed to patients with atrial fibrillation to
prevent strokes~\citep{hai19}. The optimal dosage varies considerably across
genetic, demographic, and clinical differences~\citep{Bastani15}. The Warfarin
dataset \citep{hai19} contains the optimal dosage of Warfarin for $4,788$
patients, which were found via trial and error by physicians. Using a
17-dimensional context vector on patient-specific demographics, medical
history, and genetic markers, we construct a contextual bandit benchmark where
the action space is a uniformly discretized dosage levels, and rewards are
given by absolute deviation from the optimal dosage. We present results for 20
discretized dosage levels, but as we shown in
Section~\ref{section:experiment_details}, we observe even bigger latency gains
for 50 discretized dosage levels. We present results where we reshuffle
contexts for each trial, but again find similar results when $50,000$
contexts are re-sampled each trial.

Finally, we focus on a real-world \textbf{video upload transcoding
  application}, where we study a video upload system for a leading social
network platform receiving millions of upload requests on \emph{mobile
  devices} (see Example~\ref{example:video_transcoding}).  The goal is to
preserve high quality as much as possible while ensuring upload reliability
constraints are met.  We have access to a 38-dimensional context representing
information about the video file (e.g. the raw bitrate, resolution, and file
size) and the network connection (e.g. connection type, download bandwidth,
country). There are 7 actions corresponding to a unique (resolution, bitrate)
pairs. The actions are ranked ordered in terms of quality: action $i$ yields a
video with higher quality than action $j$ if and only if $i \geq j$. If
successful, the reward for a successful upload is a positive and monotonically
increasing function of the action. The reward for a failed upload is 0.

We evaluate the performance of different contextual bandit algorithms using
the unbiased, offline, policy evaluation technique proposed by
\citet{Li11}. The method evaluates a contextual bandit algorithm by performing
rejection sampling on a stream of logged observation tuples of the form
$(S_t, A_t, R_t)$ collected under a uniform random policy. Specifically, the
observed tuple is rejected if the logged action does not match the action
selected by the algorithm being evaluated. Our dataset contains 8 million
observations logged under a uniform random policy. We evaluate each algorithm
using the stream of logged data until each algorithm has ``observed"
$50,000$ \emph{valid} examples.

Our offline evaluation is not meant to suggest offline learning is a valid
substitute for online learning algorithms. The cost of randomization and the
high level of nonstationarity in the system makes online learning algorithms
necessary. We use offline evaluations as an empirically rigorous scientific
benchmark that supports and validates our methodological development.  Our
offline dataset is generated by a particular vertical product, and provided
the empirical evidence needed to invest significant resources in implementing
the algorithm across multiple products.  As the final evaluation, we ran a
randomized controlled study as described in the introduction, and observed
significant improvements in video quality and topline business metrics (watch
time).

\begin{figure*}[t!]
\centering
    \begin{subfigure}[b]{0.5\textwidth}
        \centering
        \includegraphics[width=.8\columnwidth]{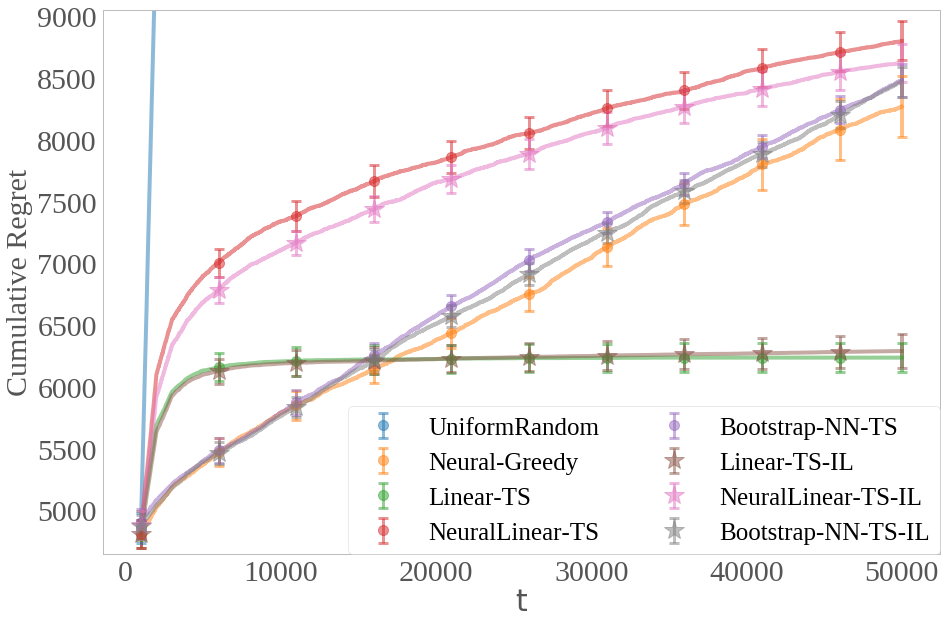}
        \caption{\label{fig:mushroom_regret} Cumulative regret on Mushroom dataset}
    \end{subfigure}%
    ~ 
    \begin{subfigure}[b]{0.5\textwidth}
        \centering
          \includegraphics[width=.8\columnwidth]{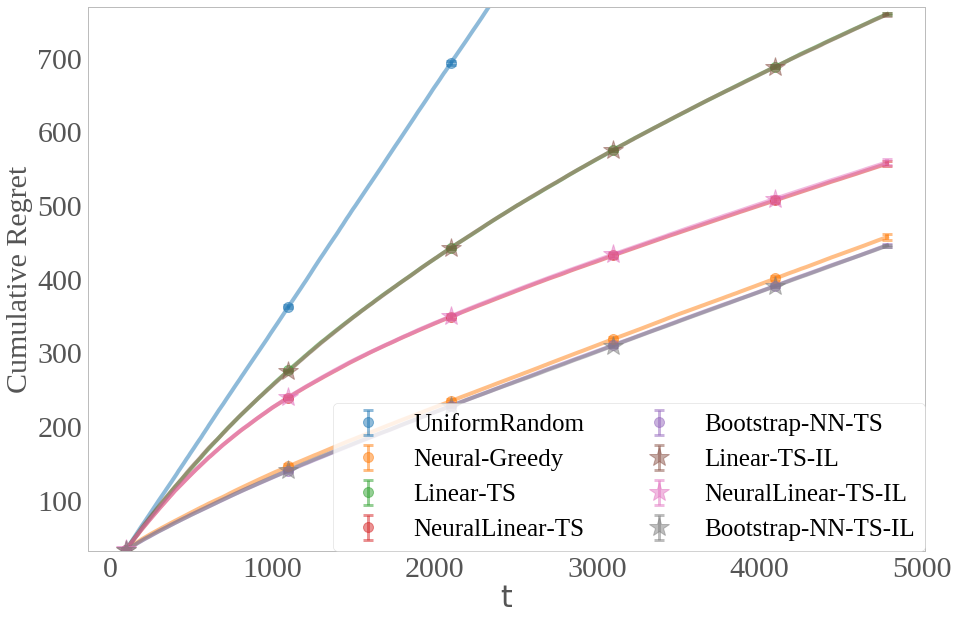}
          \caption{\label{fig:warfarin_regret} Cumulative regret on Warfarin
    dataset}
    \end{subfigure}\\
    \begin{subfigure}[b]{0.5\textwidth}
        \centering
    \includegraphics[width=.8\columnwidth]{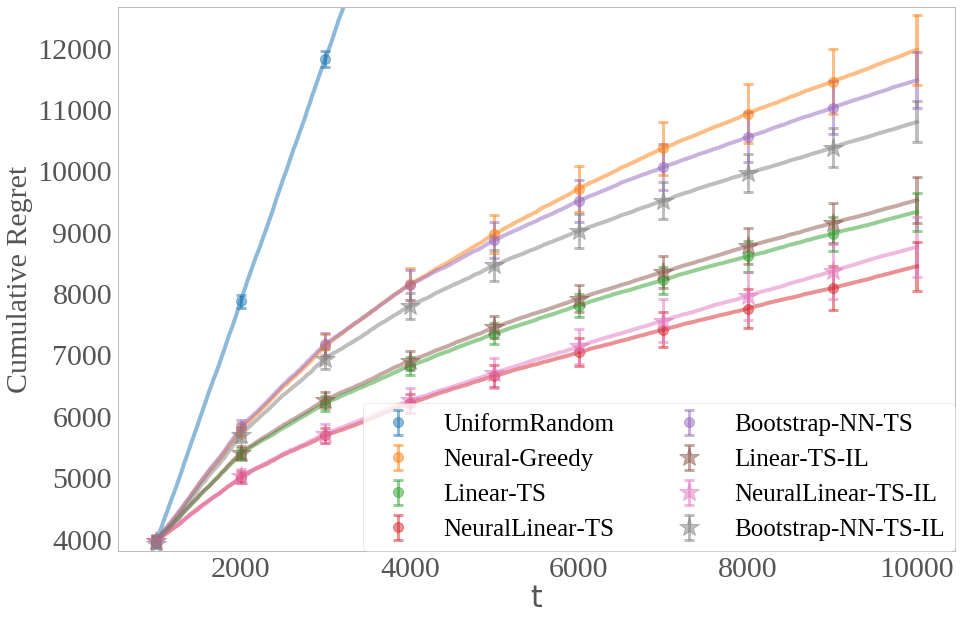}
    \caption{\label{fig:wheel_regret} Cumulative regret on Wheel bandit}
    \end{subfigure}%
    ~
    \begin{subfigure}[b]{0.5\textwidth}
        \centering
    \includegraphics[width=.8\columnwidth]{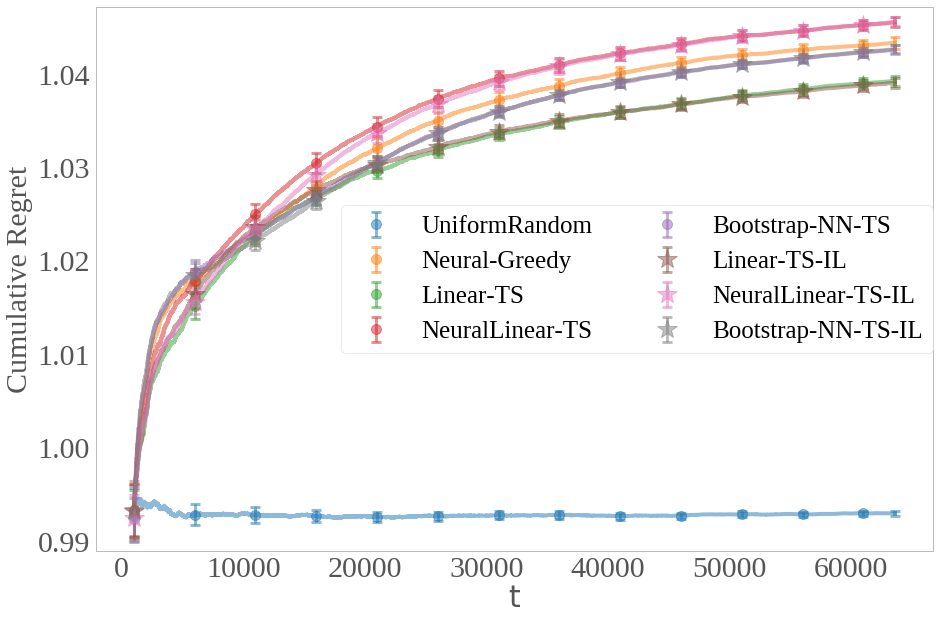}
    \caption{\label{fig:video_uploads_mean_reward} Running average of rewards for video transcoding}
    \end{subfigure}
    \caption{We report mean cumulative regret (or running average of rewards for video transcoding), alongside two standard errors over 50 trials (100 trials  for the Wheel bandit,  due to rarity of large rewards).}
    \label{fig:regret_evaluation}
\end{figure*}

\paragraph{Algorithms and evaluation}
For all experiments, we consider models previously found to perform the best
in a broad range of benchmark problems, as reported by~\citet{RiquelmeTuSn18}
in their extensive empirical experiments.  \textbf{\textsc{Linear-TS}} uses an
exact Bayesian linear regression to model the reward distribution for each
action $a$ independently. This policy evaluates the exact posterior under the
assumption that the data for action $a$ were generated from the linear
function: $r_a = \bm{s}^T \bm{\theta}_a + \varepsilon $ where
$\varepsilon \sim \mathcal N (0, \sigma_a^2)$. For each action, we
independently model the joint distribution,
$P(\bm{\theta}, \sigma^2) = P(\bm{\theta} | \sigma^2)P(\sigma^2)$ as a
normal-inverse-gamma distribution which allows for tractable posterior
inference (see Appendix \ref{subsection:algorithms} for closed form
expressions). \textbf{\textsc{NeuralLinear-TS}} models rewards using a neural
network with two 100-unit hidden layers and ReLU activations, but discards the
last linear layer and uses the last hidden layer $\bm{\phi}(\bm{s})$ as the
feature representation for a Linear-TS policy. The neural network takes the
context as input, and predicts the reward for each action. The parameters of
the neural network are shared for all actions and are learned independently of
the Bayesian linear models.  \textbf{\textsc{Bootstrap-NN-TS}} trains multiple
neural networks on bootstrapped observations and randomly samples a single
network to use for each decision. For all of the aforementioned TS policies,
\textbf{\textsc{TS-IL}} denotes their imitated counterpart. We use a
fully-connected neural network to parameterize the policy $\pi^m$ in the
imitation learning problem~\eqref{eqn:imitation}. The policy representation
has two hidden layers with 100 units each, hyperbolic tangent activations on
the hidden layers, and a soft-max activation on the output layer to predict a
the conditional distribution $P(a|\bm s)$ for all $a \in \mathcal A$. We
compare (batch) Thompson sampling and its imitation counterparts against two
additional benchmarks: a random policy (\textbf{\textsc{UniformRandom}}) and a
greedy policy that uses a feed-froward neural network to model rewards
(\textbf{\textsc{Neural-Greedy}}).

Policies are updated every $1000$ examples (except for the Warfarin problem,
where we use update policies every 100 examples due to the small size of the
dataset) and are initialized using a uniform random policy before the first
batch update. Formally, the mapping $\prev(t)$ is specified in the
definition~\eqref{eqn:batch-size}, with batch size $B = 1000~\mbox{or}~100$.
We detail our hyperparameter choices in
Section~\ref{section:experiment_details}: following extensive evaluations
by~\citet{RiquelmeTuSn18}, we use their proposed settings for Thompson
sampling.

\begin{table*}[t!]
\centering
\caption{\label{table:latency} Decision-making latency in milliseconds. All latency measurements were made on a Intel Xeon
E5-2680 v4 @ 2.40GHz CPU with 32-bit floating point precision. For each
latency measurement, action generation is repeated 100K times and the mean
latency and its 2-standard errors are reported.}
\begin{small}
\begin{sc}
\begin{tabular}{lcccc}
\toprule
& Mushroom & Wheel & Video Transcode & Warfarin\\
\midrule
UniformRandom & $0.040 ~(\pm 0.000)$ & $0.039 ~(\pm 0.000)$ & $0.040 ~(\pm 0.000)$ & $0.040 ~(\pm 0.000)$\\
Neural-Greedy & $0.242 ~(\pm 0.001)$ & $0.228 ~(\pm 0.001)$ & $0.231 ~(\pm 0.001)$ & $0.232 ~(\pm 0.000)$\\
Linear-TS & $0.715 ~(\pm 0.001)$ & $1.142 ~(\pm 0.001)$ & $1.575 ~(\pm 0.002)$ & $3.963 ~(\pm 0.002)$\\
NeuralLinear-TS & $0.826 ~(\pm 0.001)$ & $1.492 ~(\pm 0.001)$ & $1.931 ~(\pm 0.002)$ & $4.814 ~(\pm 0.004)$\\
Bootstrap-NN-TS & $0.235 ~(\pm 0.001)$ & $0.235 ~(\pm 0.001)$ & $0.236 ~(\pm 0.001)$ & $0.226 ~(\pm 0.001)$\\
Linear-TS-IL & $0.184 ~(\pm 0.001)$ & $0.178 ~(\pm 0.000)$ & $0.169 ~(\pm 0.000)$ & $0.175 ~(\pm 0.000)$\\
NeuralLinear-TS-IL & $0.186 ~(\pm 0.000)$ & $0.179 ~(\pm 0.001)$ & $0.169 ~(\pm 0.000)$ & $0.175 ~(\pm 0.000)$\\
Bootstrap-NN-TS-IL & $0.190 ~(\pm 0.001)$ & $0.178 ~(\pm 0.000)$ & $0.175 ~(\pm 0.000)$ & $0.179 ~(\pm 0.001)$\\
\bottomrule
\end{tabular}
\end{sc}
\end{small}
\end{table*}

In Figure~\ref{fig:regret_evaluation}, we show that each
\textsc{\small{TS-IL}} method achieves performance comparable to its
corresponding vanilla \textsc{\small{TS}} algorithm on all benchmark
problems. We evaluate the cumulative performance at time steps along the
entire learning curve, and observe that each \textsc{\small{TS-IL}} policy
consistently matches its corresponding \textsc{\small{TS}} policy over time.

(Approximate) Bayesian inference often requires a substantial amount of
compute and memory.  We evaluate decision-time latency and time complexity for
the specific models being considered, but note that the latency and complexity
may be even greater under inference schemes not considered here. We define
\textit{decision time latency} as the time required for a policy to select an
action when it is queried. While \textsc{Bootstrap-NN-TS} achieves low
prediction latency, it requires storing many replicates of the neural network
and can significantly increase the memory footprint. On low-end mobile devices, such memory requirements can be prohibitive, limiting the applicability of methods based on bootstrapping; our imitation methods offer a practical and effective alternative.

Table \ref{table:latency} shows that the imitation policies
(\textsc{\small{TS-IL}}) have significantly lower decision time latency
compared to \textsc{\small{TS}} algorithms, often by \emph{over an order of
  magnitude} on problems with larger action spaces (Warfarin and video upload
transcoding). This is because generating an action under the vanilla
\textsc{\small{TS}} policies requires drawing a sample from the joint
posterior $P(\bm \theta_a, \sigma_a^2)$ for each of the actions $a$, which is
quadratic with respect to the context dimension for \textsc{\small{Linear-TS}}
or the size of the last hidden layer for \textsc{\small{NeuralLinear-TS}}.  On
the other hand, \textsc{\small{TS-IL}} simply requires a forward propagation
through the policy network and a sample from multinomial sample, both of which
are exceedingly cheap. In Section~\ref{section:time_complexity}, we provide a
detailed discussion of runtime and memory complexity, including those for
alternative model choices.


%% file: regret.tex
\section{Imitation controls regret}
\label{section:imitation-controls-regret}

To understand the large practical gains we see in our numerical experiments
and randomized controlled study, we now provide some basic theoretical
analyses.  When the imitation policy generates actions
(Algorithm~\ref{alg:imitation}), the observations used to update the posterior
are different from what the batch Thompson sampling policy would have
generated. In this sense, our imitation algorithm does not emulate the batch
on-policy Thompson sampling policy, but rather simply mimics its
\emph{off-policy} variant where posterior updates are based on the history
generated by the imitation policy. In this section, we show how off-policy
imitation is sufficient to achieve Bayes regret bounds available for batch
\emph{on-policy} Thompson sampling~\citep{RussoVa14c}, up to the sum of
single-step imitation errors. In particular, our results guard against
potential exponential compounding of errors that stem from imitating the
off-policy variant of batch Thompson sampling.

We show that minimizing the KL divergence~\eqref{eqn:imitation} controls the
Bayes regret of the imitation algorithm, justifying the the imitation learning
loss~\eqref{eqn:imitation} as a valid objective.  First, we relate the
performance of our imitation policy with that of the batch off-policy Thompson
sampler~\eqref{eqn:bots} and show batch off-policy Thompson sampling admits a
Bayes regret decomposition similar to that for on-policy Thompson sampling
(Section~\ref{section:regret-decomposition}).  Building on this observation,
we use similar proof techniques for proving Bayes regret bounds on batch
on-policy Thompson sampling to provide guarantees for our imitation policy. We
substantiate our results in scenarios where batch Thompson sampling is known
to provide strong regret bounds (Section~\ref{section:regret-bound}).

\subsection{Regret decomposition}
\label{section:regret-decomposition}

Since our imitation learning problem~\eqref{eqn:imitation} approximates batch
\emph{off-policy} Thompson sampling, a pessimistic view is that any small
deviation between the imitation and Thompson sampling policy can exacerbate
over time. A suboptimal sequence of actions taken by the imitation policy may
deteriorate the performance of the batch off-policy Thompson sampling
policy~\eqref{eqn:bots} updated based on this data, compared to its on-policy
counterpart updated based on data collected by itself. Since the imitation
policy again mimics this batch off-policy Thompson sampler, this may lead to a
negative feedback loop in the worst-case.  Our analysis precludes such
negative cascades when outcomes are averaged over the prior $\prior$: the
Bayes regret of the imitation policy is comparable to that of the best batch
UCB algorithm, up to only the sum of expected discrepancy between the batch
off-policy Thompson sampling policy and the imitation learner at each period.
In particular, imitation error at each period does not affect the Bayes regret
linearly in $T$ as our worst-case intuition suggests, but rather only as a
one-time approximation cost. The single-period imitation error can be
controlled using cheap unsupervised contexts as we demonstrated in
Section~\ref{section:generalization}.

The Bayes regret suffered under the batch off-policy Thompson sampler is a
counterfactual quantity as only the imitation policy interacts with the
environment. Nevertheless, the fictitious quantity serves an important role in
our analysis. Our starting point is that an batch off-policy Thompson sampler
enjoys a Bayes regret decomposition similar to \emph{sequential, on-policy}
Thompson sampling. Since the off-policy nature of the policy does not affect
the Bayes regret decomposition, we are able to bound the Bayes regret of the
batch off-policy Thompson sampler using proof techniques developed for
\emph{sequential, on-policy} Thompson sampling~\citep{RussoVa14c}.



Before giving a formal result, we first summarize our approach, which builds
on the insights of~\citet{RussoVa14c}. We connect the performance of our
imitation policy to that of batch off-policy Thompson sampling and in turn
relate the latter method's Bayes regret to that of the \emph{best} batch UCB
algorithm. Since a similar approach also provides Bayes regret bounds for
batch on-policy Thompson sampling, our imitation policy enjoys comparable
Bayes regret, up to the sum of single-period imitation errors. Let
$U_{t}(\cdot; H_{\prev(t)}, S_t): \actionset \to \R$ be a sequence of batch
upper confidence bounds, constructed using only data collected until the most
recent batch $H_{\prev(t)}$. Let $A_t^{\rm BUCB}$ be the action taken by the
batch UCB policy (BUCB)
\begin{equation*}
  A_t^{\rm BUCB} \in \argmax_{a \in \actionset} ~U_t(a; H_{\prev(t)}, S_t).
\end{equation*}
Recalling the optimal action
$A_t\opt \in \argmax_{a \in \mc{A}} f_{\theta}(a, S_t)$, a typical argument
for bounding the regret of a BUCB algorithm proceeds by noting that since
$U_t(A_t^{\rm BUCB}; H_{\prev(t)}, S_t) \ge U_t(A_t\opt; H_{\prev(t)}, S_t)$,
\begin{align*}
  f_\theta(A_t\opt, S_t) - f_{\theta}(A_t^{\rm BUCB}, S_t) 
  \le  f_{\theta}(A_t\opt, S_t) - U_t(A_t\opt; H_{\prev(t)}, S_t) 
   + U_t(A_t^{\rm BUCB}; H_{\prev(t)}, S_t) - f_{\theta}(A_t^{\rm BUCB}, S_t).
\end{align*}
Taking expectations and summing over $t =1, \ldots, T$,
$\bayesregret{T}{\{\pi^{\rm BUCB}_t\}_{ t \in \N}}$ is bounded by
\begin{align*}
  \sum_{t=1}^T \E[f_{\theta}(A_t\opt, S_t) - U_t(A_t\opt; H_{\prev(t}, S_t)] 
  + \sum_{t=1}^T \E[U_t(A_t^{\rm BUCB}; H_{\prev(t)}, S_t)
    - f_{\theta}(A_t^{\rm BUCB}, S_t)].
\end{align*}
If the upper confidence bound property holds uniformly over the actions so
that $U_t(a; H_{\prev(t)}, S_t) \ge f_{\theta}(a, S_t)$ for all
$a \in \actionset$ with high probability, the first term in the above regret
decomposition can be seen to be nonpositive.  To bound the second term, a
canonical proof notes each upper confidence bound is not too far away from the
population mean $f_{\theta}(A_t^{\rm BUCB}, S_t)$.~\citet{RussoVa14c}'s key
insight was that (sequential on-policy) Thompson sampling admits an analagous
Bayes regret decomposition as above, but with respect to \emph{any} UCB
sequence. This allows leveraging arguments that bound the (frequentist) regret
of a UCB algorithm to bound the Bayes regret of Thompson sampling. Since the
Bayes regret decomposition for Thompson sampling holds for \emph{any} UCB
sequence, the performance of Thompson sampling enjoys Bayes regret guarantees
of the best UCB algorithm.

By connecting the performance of our imitation policy to that of \emph{batch
  off-policy Thompson sampling}, we show that a similar Bayes regret
decomposition can be leveraged despite its off-policy nature. Recall that we
denote $\bar{A}_t \sim \bar{\pi}_{\prev(t)}(\cdot \mid S_t)$, the action
generated by the batch off-policy Thompson sampler.  See
Section~\ref{section:proof-imitation-controls-decomposition} for the proof of
the following result. 
\begin{lemma}
  \label{lemma:imitation-controls-decomposition}
  Let $\{\pi_{\prev(t)}\}_{t \in \N}$ and $U_t(\cdot; H_{\prev(t)}, S_t)$ be
  any sequence of batch policies and UCBs (adapted to the history
  $H_{\prev(t)}$). If
  $\E[\sup_{a \in \mc{A}} f_{\theta}(a, S)^2] \eqdef L^2 < \infty$,
  {\small
  \begin{align}
     \bayesregret{T}{\{\pi_{\prev(t)}\}_{t \in \N}}
      & \le \underbrace{\sum_{t=1}^T  \E[f_{\theta}(A_t\opt, S_t) - U_t(A_t\opt; H_{\prev(t)}, S_t)] 
      + \sum_{t=1}^T \E[U_t(\bar{A}_t; H_{\prev(t)}, S_t) - f_{\theta}(\bar{A}_t, S_t)]}_{\small \mbox{(a): regret decomposition for any batch UCB algorithm}}  \nonumber \\
    & \qquad + \underbrace{L \sum_{t=1}^T \sqrt{\half \E\left[
      \dklpolicy{\bar{\pi}_{\prev(t)}}{\pi_{\prev(t)}}{S_t} \right]}}_{\small \mbox{(b): imitation error}}.
       \label{eqn:regret-decomposition}
  \end{align}
  }
\end{lemma}

The Bayes regret decomposition~\eqref{eqn:regret-decomposition} shows that
performance analysis of any batch UCB algorithm can characterize the regret of
our imitation policy. In this sense, the imitation policy achieves regret
comparable to the \emph{optimal} batch UCB algorithm, up to the sum of
single-period imitation errors. As we detail shortly in a general modeling
scenario based on contextual Gaussian processes, term (a) can be bounded using
canonical batch UCB proofs.  Term (b) can be controlled by our imitation
learning algorithm (Algorithm~\ref{alg:imitation}) and its empirical
approximation as seen in Section~\ref{section:generalization}. Although this
term scales as $O(T/\sqrt{N})$, we argue that the seemingly linear dependence
on $T$ is not of material concern. In large-scale internet applications, the
number of unsupervised contexts $N$ is very large as they can simply be read
off of a database of user information ($N\approx 10-100M$). The number of
policy updates $T$ is often orders of magnitude smaller (hundreds) in a
typical product lifecycle due to operational challenges in deploying a policy.
Thus, the term (b) can be made relatively small using big datasets and
powerful overparameterized imitation models using the results in
Section~\ref{section:generalization}.

The fact that we are studying Bayes regret, as opposed to the frequentist
regret, plays an important role in the above decomposition. We conjecture that
in the worst-case, imitation error at any period (and consequently suboptimal
exploration) can each linearly compound over time, leading to a prohibitive
quadratic dependence on $T$.  It remains open whether specific problem
structures can provably preclude such negative feedback loops uniformly over
$\theta$.


%% file: contextual-gp.tex


\subsection{Regret bounds for contextual Gaussian processes}
\label{section:regret-bound}

We now show concrete performance guarantees for our imitation algorithm by
using instance-independent (gap-independent) Bayes regret bounds for batch
off-policy Thompson sampling. Despite its counterfactual nature, the
decomposition~\eqref{eqn:regret-decomposition} enables us to control it using
identical proof techniques for controlling the Bayes regret of batch
\emph{on-policy} Thompson sampling. This program allows us control over the
term $(a)$ in the decomposition~\eqref{eqn:regret-decomposition}.

We consider a general setting where the mean reward function
$(a, s) \mapsto f_{\theta}(a, s)$ can be modeled as a sample path of a
Gaussian process, with potentially continuous action and context spaces.
Formally,
we assume that $(a, s) \mapsto f_{\theta}(a, s)$ is sampled from a Gaussian
process on $\mc{A} \times \mc{S}$ with mean function $\mu(a, s)$ and
covariance function (kernel)
\begin{equation*}
  \Sigma((a, s), (a',s')) \defeq
  \E[(f_{\theta}(a, s) - \mu(a, s)) (f_{\theta}(a', s') - \mu(a', s'))].
\end{equation*}
We assume that the decision maker observes rewards
\begin{equation*}
  R_t = f_{\theta}(A_t, S_t) + \epsilon_t,
\end{equation*}
where the noise $\epsilon_t \simiid N(0, \sigma^2)$ are independent of
everything else.  Given these rewards, we are interested in optimizing the
function $a \mapsto f_{\theta}(a, S_t)$ for each observed context $S_t$ at
time $t$.

Modeling mean rewards as a Gaussian process is advantageous since we can
utilize analytic formulae to update the posterior at each step.  Since
$f_{\theta}(a, s)$ follows a Gaussian process, its posterior is also a
Gaussian process with mean and variance is given by
\begin{align*}
  \mu_t(a, s)
  & \defeq \E[f_{\theta}(a, s) \mid H_t] = \Sigma_t(a, s)^\top (K_t + \sigma^2 I)^{-1} \vec{R}_t, \\
  \sigma_t^2(a, s)
  & \defeq \var(f_{\theta}(a, s) \mid H_t)
    = \Sigma((a, s), (a, s)) - \Sigma_t(a, s)^\top (K_t + \sigma^2 I)^{-1} \Sigma_t(a, s)
\end{align*}
where $\Sigma_t(a, s) \defeq [\Sigma((A_j, S_j), (a, s))]_{1 \le j \le t-1}$,
$K_t \defeq [k((A_i, S_i), (A_j, S_j))]_{1\le i, j\le t-1}$ and
$\vec{R}_t = [R_j]_{1\le j \le t-1}$.  For large-scale applications, we can
parameterize our kernels by a neural network and leverage the recently
developed interpolations techniques to perform offline posterior
updates~\citep{WilsonNi15, WilsonDaNi15, WilsonHuSaXi16}.

We leverage regret bound techniques for batch UCB
algorithms~\citep{DesautelsKrBu14} to bound the term $(a)$ in the Bayes regret
decomposition~\eqref{eqn:regret-decomposition}. This term is controlled by the
maximal amount of information on the optimal action that can be gained after
$T$ time steps. Recall the definition of (conditional) mutual information
between two random vectors
\begin{equation*}
  I(Z, Y) \defeq \dkl{P_{Z, Y}}{P_Z \times P_Y}
  ~~\mbox{and}~~I(Z, Y \mid W) \defeq \dkl{P_{Z, Y \mid W}}{P_{Z \mid W} \times P_{Y\mid W}}
\end{equation*}
We define the maximal possible
information gain after $T$ time steps as
\begin{equation*}
  \gamma_T \defeq \sup_{\mc{X} \subseteq \mc{A} \times \mc{S}: |\mc{X}| = T}
  I(\vec{R}_{\mc{X}}, f_{\mc{X}})
\end{equation*}
where $\vec{R}_{\mc{X}} = \{f_{\theta}(x) + \epsilon_x\}_{x \in \mc{X}}$ and
$f_{\mc{X}} = \{f_{\theta}(x)\}_{x \in \mc{X}}$.  For popular Gaussian and Matern
kernels,~\citet{SrinivasKrKaSe12} has shown that the maximal information gain
can be bounded explicitly; we summarize these bounds shortly.

Due to the batched nature of Algorithm~\ref{alg:imitation}, we further need to
control the maximal information gain in a single batch, assuming that the
(time-varying) batch size is uniformly bounded by some constant $\batch$.
\begin{assumption}
  \label{assumption:batch-info}
  Let $\prev(t+1) - \prev(t) \le \batch$ for $1 \le t \le T$ and let
  $\eta_{\batch}$ be a constant satisfying
  \begin{equation}
    \max_{\mc{X} \subseteq \mc{A} \times \mc{S}: |\mc{X}| \le \batch}
    I\left(\vec{R}_{\mc{X}}, f_{\mc{X}} \mid \vec{R}_{\prev(t)}\right)
    \le \half \log(\eta_{\batch})
    ~~\mbox{for all}~ 1 \le t \le T
  \end{equation}
  where $\vec{R}_{\mc{X}} = \{f_{\theta}(x) + \epsilon_x\}_{x \in \mc{X}}$,
  $\vec{R}_{\prev(t)} = \{ R_1, \ldots, R_{\prev(t)-1}\}$, and
  $f_{\mc{X}} = \{f_{\theta}(x)\}_{x \in \mc{X}}$.
\end{assumption}

For a compact action space $\mc{A} \subset \R^d$, term $(a)$ in the
decomposition~\eqref{eqn:regret-decomposition} is bounded by
$O\left(\sqrt{d \eta_{\batch} \gamma_t T (\log T)^{d}}\right)$.  Our proof
relies on the batch upper confidence bound
\begin{equation}
  \label{eqn:ucb}
    U_t(a; H_{\prev(t)}, s) \defeq
  \mu_{\prev(t)}(a, s) + \sqrt{\beta_t} \sigma_{\prev(t)}(a, s)
  ~~\mbox{where}~~\beta_t = 2 \log((T^4 rd)^d T^2)
\end{equation}
We use $L_f$ to denote the (random) Lipschitz constant of
the map $a \mapsto f_{\theta}(a, s)$
\begin{equation*}
  L_f \defeq \sup_{s \in \mc{S}} \sup_{a, a' \in \mc{A}}
  \frac{|f_{\theta}(a, s) - f_{\theta}(a', s)|}{\norm{a - a'}_1}.
\end{equation*}
Standard
arguments from Gaussian process theory show $\E[L_f^2] < \infty$ holds
whenever $\mu(\cdot)$ and $\Sigma(\cdot, \cdot)$ are 4 times continuously
differentiable~\cite[Theorem 5]{GhosalRo06}.
\begin{theorem}
  \label{theorem:gp}
  For $\mc{A} \subseteq [0, r]^d$ for some $r > 0$, let
  Assumption~\ref{assumption:batch-info} hold. Assume that
  \begin{equation*}
    c_1 \defeq \sup_{a \in \mc{A}, s \in \mc{S}} |\mu(a, s)| < \infty,
    ~~~~c_2 \defeq \sup_{a, a' \in \mc{A}, s, s' \in \mc{S}} \Sigma((a, s), (a', s')) < \infty,
  \end{equation*}
  and let
  $L^2 \defeq \E\left[ \sup_{a \in \mc{A}, s\in \mc{S}} f_{\theta}(a,
    s)^2\right]$ as before.
  If $\E[L_f^2] < \infty$, there is a universal
  constant $C > 1$ such that
  \begin{align*}
    & \bayesregret{T}{\pi}
      \le C \E[L_f] + C c_2
      + C d \log (rd) \left(c_1 \sqrt{\E[L_f] } + c_3 \sqrt{\E[L_f^2]}\right) \\
    &  \qquad \qquad  + \left(T \eta_{\batch}  \gamma_T \frac{d \log T + d\log rd}{\log(1+\sigma^{-2})} \right)^{1/2} 
      + (L + \sqrt{c_2 \beta_T}) \sum_{t=1}^T \sqrt{\half
      \E\left[ \dklpolicy{\bar{\pi}_{\prev(t)}}{\pi_{\prev(t)}}{S_t} \right]}.
  \end{align*}
\end{theorem}
\noindent See Section~\ref{section:proof-gp} for the proof.  

\paragraph{Bounds on $\gamma_T$}
To obtain concrete bounds on the maximal information gain $\gamma_T$, we focus on
the popular and flexible linear, Gaussian and Matern kernels
\begin{align*}
  \Sigma_l(x,x') & \defeq x^\top x', 
  \Sigma_g(x, x')  \defeq \exp\left( - \frac{\norm{x - x'}^2}{2l^2} \right), \\
  \Sigma_m (x, x') & \defeq
                \frac{2^{1-\nu}}{\Gamma(\nu)}
                r^\nu B_{\nu}(r)
                ~~\mbox{where}~r = \frac{\sqrt{2\nu}}{l} \norm{x -x'},
\end{align*}
where we used $B(\cdot)$ and $\Gamma(\cdot)$ to denote the Besel and Gamma
functions respectively. To ease notation, we let $\kappa$ denote the dimension
of the underlying space, and define
\begin{align*}
 \mathfrak{M}(\Sigma_l, T) \defeq \kappa \log T,
 ~~~~\mathfrak{M}(\Sigma_g, T) \defeq (\log T)^{\kappa+1},
  ~~~~\mathfrak{M}(\Sigma_m, T)
  \defeq T^{\frac{\kappa^2+\kappa}{\kappa^2+\kappa + 2\nu}}\log T.
\end{align*}
We have the following bound on $\gamma_T$ for linear, Gaussian, and Matern
kernels; the bound is a direct consequence of~\citet[Theorem 2]{KrauseOn11}
and~\citet[Theorem 5]{SrinivasKrKaSe12}.
\begin{lemma}
  \label{lemma:info-bound}
  Let $\mc{A} \subseteq \R^{d}$ and $\mc{S} \subseteq \R^{d'}$ be convex and
  compact. Let the kernel $\Sigma$ be given by
  $\Sigma((a, s), (a',s')) \defeq \Sigma_A(a, a') + \Sigma_S(s, s')$. Then,
    $\gamma_T = O\left( \mathfrak{M}(\Sigma_A, T) + \mathfrak{M}(\Sigma_S, T) + \log T
    \right)$.
\end{lemma}

\paragraph{Bounds on $\eta_{\batch}$} To control the Bayes regret of batch
off-policy Thompson sampling, it remains to control the per batch information
gain $\eta_{\batch}$.  Our development so far allows us to use techniques
developed for on-policy Thompson sampling to bound this quantity.  A naive
bound for the per batch information gain $\eta_{\batch}$ is
$\eta_{\batch} \le \exp(2\gamma_{\batch})$, which can be prohibitively large
in large batch scenarios. Towards tighter theoretical control, we use a clever
initialization scheme due to~\citet{DesautelsKrBu14}. While we conjecture that
batch Thompson sampling will perform well even without such a careful
initialization scheme, we are unable to theoretically confirm the conjecture
and leave it as future work.

We initialize our algorithm by targeting $T_{\rm init}$ users/contexts who
suffer the highest uncertainty in their reward. Considering the initialization
index set $t \in \{-T_{\rm init} +1, \ldots, 0\}$, the posterior variance does
not depend on previous rewards
\begin{align*}
  \sigma_{t}^2(a, s)
  & \defeq \var(f_{\theta}(a, s) \mid H_t)
    = \Sigma((a, s), (a, s)) - \Sigma_t(a, s)^\top (K_t + \sigma^2 I)^{-1} \Sigma_t(a, s),
\end{align*}
where
$\Sigma_(a, s) \defeq [\Sigma((A_j, S_j), (a, s))]_{-T_{\rm init} + 1 \le j
  \le t-1}$ and
$K_t \defeq [k((A_t, S_t), (A_j, S_j))]_{-T_{\rm init} + 1 \le j \le t- 1}$.
Thus, before engaging with the environment we can sequentially calculate
\begin{align*}
  (A_t^{\rm init}, S_t^{\rm init})
  \in \argmax_{a \in \mc{A}, s \in \mc{S}} \sigma^2_t(a, s)
  \qquad \mbox{for}~~t = -T_{\rm init}+1, \ldots, 0
\end{align*}
We initially target users/contexts $S_t^{\rm init}$ in the database with
actions $A_t^{\rm init}$ for $t = -T_{\rm init} +1, \ldots, 0$. Using the
history $H_t = \{S_i, A_i, R_i\}_{i=-T_{\rm init}+1}^{t-1}$, we redefine
Thompson sampling and Algorithm~\ref{alg:imitation} with initialization
data. The following result shows that this initialization procedure controls
the per batch information gain $\eta_{\batch}$.  For simplicity, we consider
combinations of linear or Gaussian kernels and define
$\bar{d} \defeq \max(d, d')$. Recalling the batch upper confidence
bound~\eqref{eqn:ucb}, the following result is a direct consequence
of~\citet[Lemma 4, Theorem 5]{DesautelsKrBu14}; an analogous bound holds for
Matern kernels, but we omit it for brevity.
\begin{proposition}
  \label{prop:init}
  Let the conditions of Theorem~\ref{theorem:gp} hold and let
  $\Sigma_A, \Sigma_S \in \{\Sigma_l, \Sigma_g\}$. Consider the initialization
  procedure described in the previous paragraph with $T_{\rm init}$ periods.
  There is a constant $C>0$ such that if we set
  $T_{\rm init} = C^{\bar{d}+1} \batch (\log \batch)^{\bar{d}+1}$, then
    \begin{align*}
  & \bayesregret{T}{\{\pi_{\prev(t)}\}_{t \in \N}}
   = \sum_{t=1}^T \E_{\theta \sim \prior}\left[ \max_{a \in \actionset} f_\theta(a, S_t) - f_\theta(A_t, S_t) \right] \\
      & \le  O\left(\exp(\bar{d}^{\bar{d}}) \sqrt{\bar{d} T (\log T)^{\bar{d} + 1}}\right)
        + (L + \sqrt{c_2 \beta_T}) \sum_{t=1}^T  \sqrt{\half
        \E\left[ \dklpolicy{\bar{\pi}_{\prev(t)}}{\pi_{\prev(t)}}{S_t} \right]}.
    \end{align*}
\end{proposition}
\noindent The first term bounds the Bayes regret of batch on-policy Thompson
sampling; in comparison, sequential on-policy Thompson
sampling~\citep{KrauseOn11} achieves Bayes regret
$O\left(\sqrt{\bar{d} T (\log T)^{\bar{d} +
      1}}\right)$.


%% file: generalization.tex
\section{Generalization guarantees for imitation learning}
\label{section:generalization}

In this section, we show that solving an empirical approximation of the
imitation problem~\eqref{eqn:mle} can control the imitation objective.  From
results in the previous section, this in turn shows that the regret can be
controlled when we have many contexts. Given i.i.d. observations of
(potentially unsupervised contexts) $S_i \simiid P_S$, we solve the empirical
approximation to the imitation problem~\eqref{eqn:mle}
\begin{align}
  \label{eqn:real-empirical-il}
  \what{m} \in \argmax_{m \in \mc{M}}
  \frac{1}{N} \sum_{i=1}^N \frac{1}{N_a} \sum_{j=1}^{N_a}
  \log \pi^m(\bar{A}_{ij} \mid S_i),
\end{align}
where we simulate actions from the batch off-policy Thompson
sampling~\eqref{eqn:bots}
\begin{equation*}
  \bar{A}_{ij} \sim \bar{\pi}_{\prev(t)}(\cdot \mid S_i)
  ~~j = 1, \ldots, N_a
\end{equation*}
for each context $S_i$. Since actions can be simulated \emph{offline} in a
parallel manner, we can efficiently generate a large number of actions $N_a$.

In what follows, we assume that our imitation model class is
\emph{well-specified}, so that there exists $m\opt \in \mc{M}$ satisfying
$\bar{\pi} = \pi^{m\opt}$, where we omitted the subscript and denote
$\bar{\pi} = \bar{\pi}_{\prev(t)}$ to ease notation. This is often a
reasonable assumption as we consider expressive model classes such as
nonparametric models involving reproducing kernel Hilbert spaces. With a
well-specified imitation model, we prove with probability at least $1-\delta$,
\begin{align}
  \label{eqn:generalization-rough}
  \E_{S \sim \P_s}\dklpolicy{\bar{\pi}}{\pi^{\what{m}}}{S}
  \lesssim \frac{1}{N} \left( \mathfrak{Comp}_N +  \log \frac{1}{\delta}\right)
  + \frac{\mathfrak{Comp}_{N, N_a}}{\sqrt{N N_a}},
\end{align}
for some complexity measures $\mathfrak{Comp}_N$ and
$\mathfrak{Comp}_{N, N_a}$ associated with the imitation model class
$\mc{M}$. Here, the notation $\lesssim$ denotes inequality up to a universal
constant. In typical internet applications, the number of unsupervised
contexts $N$ is exceedingly large, and the imitation
error~\eqref{eqn:generalization-rough} can be made vanishingly small.

The key challenges to showing the preceding result are twofold: 1) the
empirical procedure~\eqref{eqn:real-empirical-il} employs non-i.i.d. samples
$(S_i, \bar{A}_{ij})$, so standard concentration results do not apply, and 2)
the bound~\eqref{eqn:generalization-rough} scales with the ``fast rate''
$1/N$, rather than the canonical parametric rate $1/\sqrt{N}$. To overcome the
first challenge, our proof carefully derives concentration inequalities for
the two-step sampling process where nature generates $S_i \simiid \P_S$, and
for each $S_i$ we simulate $A_{ij} \simiid \bar{\pi}(\cdot \mid S_i)$ via
posterior sampling. To prove the fast rate of convergence $1/N$, we use an
elaborate localization-based proof approach~\citep{BartlettBoMe05} which
exploits the fact that the complexity of the function class
$(s, a) \mapsto \log \pi^m(a\mid s)$ may be substantially smaller on a
neighborhood of the optimum $m\opt$, compared to over the entire model space
$m \in \mc{M}$,

To formalize our arguments, recall the standard notion of Rademacher
complexity: for a fixed $\xi_1, \ldots, \xi_n$ and i.i.d.\ random signs
$\varepsilon_i \in \{-1,1\}$ (Rademacher variables) that are independent of
the $\xi_i$'s, the empirical Rademacher complexity of the class of functions
$\mc{G} \subseteq \{ g: \Xi \to \R\}$ is
\begin{equation*}
  \radcomp_n(\mc{G}) \defeq \E_{\varepsilon} \left[\sup_{g \in \mc{G}}
    \frac{1}{n} \sum_{i = 1}^n \varepsilon_i g(\xi_i) \right].
\end{equation*}
A function $\psi : \R_+ \to \R_+$ is \emph{sub-root}~\citep{BartlettBoMe05} if
it is nonnegative, nondecreasing, and $r \mapsto \psi(r) / \sqrt{r}$ is
nonincreasing for all $r > 0$. This analytic notion guarantees that any
non-constant sub-root function $\psi$ is continuous, and has a unique positive
fixed point $r\opt = \psi(r\opt)$, where $r \ge \psi(r)$ for all
$r \ge r\opt$.  Let $\psi_n : \R_+ \to \R_+$ be a sub-root upper bound on the
localized Rademacher complexity 
\begin{equation}
  \label{eqn:sub-root}
  \psi_n(r) \ge \E[ \radcomp_n(\{ g \in \mc{G}: \E[g^2]
  \le r\})].
\end{equation}
(The localized Rademacher complexity itself is sub-root.)  Fixed points of
$\psi_n$ characterize uniform concentration guarantees;
see~\citet{BartlettBoMe05} and \citet{Koltchinskii06a} for a detailed analysis
of localized Rademacher complexities.

The Rademacher complexity of the following set of functions controls the
generalization performance of the empirical imitation
model~\eqref{eqn:real-empirical-il}
\begin{align*}
  \mc{G}_1 & \defeq \left\{ s \mapsto \E_{\bar{A} \sim \bar{\pi}(\cdot \mid s)}
             \left[\log \frac{\bar{\pi}(\bar{A}\mid s)}{\pi^m(\bar{A}\mid s)}\right]:
             m \in \mc{M} \right\} \\
  \mc{G}_2(s)  & \defeq \left\{
                 a \mapsto \log \pi^m(a \mid s): m \in \mc{M}
                 \right\} \\
  \mc{G}_3  & \defeq \left\{
              (a, s) \mapsto \log \pi^m(a \mid s):  m \in \mc{M}
              \right\}.
\end{align*}
We let $r_N\opt$ be the unique fixed point of the sub-root function $\psi_n$
satisfying the bound~\eqref{eqn:sub-root} for $\mc{G} = \mc{G}_1$
\begin{equation*}
  r_N\opt = \psi_n(r_N\opt)~~\mbox{where}~~
  \psi_n(r) \ge \E[ \radcomp_n(\{ g \in \mc{G}_1: \E[g^2]
  \le r\})]
\end{equation*}
For any fixed
context $s \in \mc{S}$, using i.i.d. random signs $\varepsilon_{j}$, we
write
\begin{align*}
    \radcomp_{N}\mc{G}_2(s) \defeq
  \E_{\epsilon}\left[\sup_{m \in \mc{M}}\frac{1}{N_a} \sum_{j=1}^{N_a}  \varepsilon_{j}
  \log \pi^m(\bar{A}_{j} \mid s) \right].  
\end{align*}
For $\mc{G}_3$,
using i.i.d. random signs $\varepsilon_{ij}$ we still write
\begin{equation*}
  \radcomp_{NN_a}\mc{G}_3 \defeq
  \E_{\epsilon}\left[\sup_{m \in \mc{M}}\frac{1}{N} \sum_{i=1}^N \frac{1}{N_a} \sum_{j=1}^{N_a} \varepsilon_{ij} \log \pi^m(\bar{A}_{ij} \mid S_i) \right].
\end{equation*}

Our main result in this section shows that the imitation error of the
empirical solution~\eqref{eqn:real-empirical-il} is
$O_p\left(N^{-1} + N^{-1/2}N_a^{-1/2}\right)$. See
Section~\ref{section:proof-generalization} for the proof.
\begin{theorem}
  \label{theorem:generalization}
  Let there exist a $m\opt \in \mc{M}$ such that $\bar{\pi} = \pi^{m\opt}$. Assume
  $|\log \pi^m(a \mid s)| \le \zbound$ for all
  $a \in \mc{A}, s \in \mc{S}, m \in \mc{M}$. There is a numerical constant
  $C>0$ s.t.  with probability at least $1-2e^{-z}$
  \begin{align*}
    \E \left[ \dklpolicy{\bar{\pi}}{\pi^{\what{m}}}{S} \right]
     \le C\left( \frac{1}{\lbound}r_N\opt + \frac{\lbound t}{N}  + \sqrt{\frac{z}{N}}
      \sup_{s \in \mc{S}}\E_{\bar{A}_j \simiid \bar{\pi}
      (\cdot \mid s)}[ \radcomp_{N_a} (\mc{G}_2(s))]
      + \E[\radcomp_{NN_a}(\mc{G}_3)]\right).
  \end{align*}
\end{theorem}
\noindent For finite-dimensional model classes with bounded VC-dimension, standard
arguments bound the Rademacher complexity terms in the above theorem~\citep[Ch
2.6]{VanDerVaartWe96}. Denoting by $\vcdim(\cdot)$ the VC-dimension,
we have
\begin{align*}
  \sup_{s \in \mc{S}}\E_{\bar{A}_j \simiid \bar{\pi} (\cdot \mid s)}[
\radcomp_{N_a} (\mc{G}_2(s))] \le \lbound \sqrt{\frac{\sup_{s \in
      \mc{S}}\vcdim(\mc{G}_2(s))}{N_a}}~~~\mbox{and}~~~
  \E[\radcomp_{NN_a}(\mc{G}_3)] \le \lbound \sqrt{\frac{\vcdim(\mc{G}_3)}{N
  N_a}}.
\end{align*}
Moreover, Corollary 3.7 of \citet{BartlettBoMe05} implies that
$r_N\opt \asymp \frac{\lbound\vcdim(\mc{G}_1) \log (N /
  \vcdim(\mc{G}_1))}{N}$. Plugging these bounds in
Theorem~\ref{theorem:generalization}, we obtain the previously claimed
convergence rate~\eqref{eqn:generalization-rough}.

Due to the generality of our localized Rademacher complexity approach, we can
provide imitation guarantees  for substantially larger and more expressive
\emph{nonparametric} model classes. We consider a reproducing kernel Hilbert
space (RKHS) $\mc{H}$ defined over a kernel
$k: \Xi \times \Xi \to \R_+$~\citep{BerlinetAg04}. For such nonparametric
models, standard covering number bounds are loose~\citep{Kuhn11}, while
localized arguments can still provide fast
concentration~\citep{Mendelson03}. Consider a RKHS with norm
$\norm{\cdot}_{\mc{H}}$ and evaluation kernel $k(\cdot, \cdot)$. Mercer's
theorem~\citep{CristianiniSh04} states that the integral operator
$T_k : L^2(\Xi, P) \to L^2(\Xi, P), T_k(h)(\xi) = \int h(\xi')K(\xi, \xi')
dP(\xi')$ is compact, and we have the eigenbasis expansion
$k(\xi, \xi') = \sum_{j = 1}^\infty \lambda_j \phi_j(\xi) \phi_j(\xi')$ where
$\lambda_j$ are eigenvalues of $T$ sorted in decreasing order and $\phi_j$
give an orthonormal decomposition in $L^2(\mc{Z}, P)$.

Let $k_{\mc{S}}: \mc{S} \times \mc{S} \to \R_+$ and
$k_{\mc{A}}: \mc{A} \times \mc{A} \to \R_+$ be kernels on $\mc{S}$ and
$\mc{A}$ respectively, and let us denote by $\B_{\mc{S}}$ and $\B_{\mc{A}}$
the unit ball in the respective RKHS's. The kernels $k_{\mc{S}}$ and
$k_{\mc{A}}$ induce a RKHS over functions on $\mc{S} \times \mc{A}$ formed
with the kernel $k((s, a), (s', a')) = k_{\mc{S}}(s, s') + k_{\mc{A}}(a, a')$;
we denote the unit ball in this space by $\B_{\mc{S}\times\mc{A}}$. For
simplicity, we assume that the function classes $\mc{G}_1, \mc{G}_2(s)$, and
$\mc{G}_3$ belong in a unit ball in appropriately defined RKHS's
\begin{align*}
  \mc{G}_1 \subset \B_{\mc{S}},~~\mc{G}_2(s) \subset \B_{\mc{A}}~\mbox{for all}~s\in\mc{S},~~\mc{G}_{3} \subset \B_{\mc{S}\times\mc{A}}.
\end{align*}
For RKHS-based models, the rate of decay of the eigenvalues of
$T_{k_{\mc{S}}}$ controls the rate of convergence in
Theorem~\ref{theorem:generalization}.  For example, eigenvalues of the popular
Gaussian kernel $k(\xi, \xi') = \exp(-\half \ltwo{\xi - \xi'}^2)$ decay
exponentially fast $\lambda_j \lesssim
e^{-j^2}$~\citep{Mendelson03}. Eigenvalues of kernel operators $T_k$ for
Sobolev spaces~\citep{BirmanSo67,Gu02} decay polynomially fast
$\lambda_j \lesssim j^{-2 \beta}$, where $\beta > \half$ is the smoothness
level. e.g., in 1-dimension, the first-order Sobolev kernel
$k(\xi, \xi') = 1 + \min\{\xi, \xi'\}$ where $\beta = 1$ generates RKHS of
Lipschitz functions.  We prove the below corollary in
Section~\ref{section:proof-kernel}.
\begin{corollary}
  \label{corollary:kernel}
  Assume
  $\sup_{s \in \mc{S}} k_{\mc{S}}(s, s) + \sup_{a \in \mc{A}} k_{\mc{A}}(a, a)
  \le B$ for some $B>0$. If the eigenvalues of $T_{k_{\mc{S}}}$ decay as
  $\lambda_j \lesssim e^{-j^2}$, there is a numerical constant $C>0$ s.t. with
  probability at least $1-2e^{-z}$
  \begin{align*}
    \E \left[ \dklpolicy{\bar{\pi}}{\pi^{\what{m}}}{S} \right]
    \le C\frac{\lbound z + \sqrt{\log N}}{N} + C\lbound B \sqrt{\frac{z+1}{NN_a}}.
  \end{align*}
  If the eigenvalues of $T_{k_{\mc{S}}}$ decay as $\lambda_j \lesssim j^{-2\beta}$
  for some $\beta > 1/2$, then there is another numerical constant $C>0$ such
  that with probability at least $1-2e^{-z}$
  \begin{align*}
    \E \left[ \dklpolicy{\bar{\pi}}{\pi^{\what{m}}}{S} \right]
    \le C\left( \frac{\lbound z}{N} + N^{\frac{-2\beta}{2\beta+1}} + \lbound B \sqrt{\frac{z+1}{NN_a}}\right).
  \end{align*}
\end{corollary}


%% file: discussion.tex
\section{Discussion}

In this paper, we used imitation learning to operationalize Thompson sampling,
allowing it to scale to applications where latency and software complexity are
of concern.  We demonstrated that imitation learning provides a simple,
practical, and efficient method with desirable regret properties.  By
distilling the Thompson sampling policy into easy-to-deploy explicit policy
representations (e.g. neural networks), we allow state-of-the-art Bayesian
approaches to be used in contextual bandit problems. We hope that this work
facilitates applications of modern deep learning-based Bayesian approaches to large-scale
contextual bandit problems.

While we have empirically evaluated two types of Bayesian models, our
framework is compatible with any type of probabilistic model. For example,
practitioners may utilize domain knowledge to develop grey-box models (see
e.g., \citet{SchwartzBrFa17}).  Such models are simple to implement in
probabilistic programming
languages~\citep{carpenter2017stan,bingham2018pyro,tran2019ppl}, but
challenging and inefficient to deploy. Our imitation framework can allow ease
of deployment for these models while maintaining a comparable level of
performance. Although we restricted attention to contextual bandits problems,
an interesting research direction is to extend these methods to combinatorial
ranking problems~\citep{Wang18, dimakopoulou19}, where computational savings
of distillation may be even larger.  Extending our imitation framework to
reinforcement learning problems will also likely yield fruit.






%% file: wasserstein.tex
\section{Imitation learning with Wasserstein distances}
\label{section:wass}

When actions can be naturally embedded in a continuous space, we may want to
measure closeness between the imitation and TS policy by incorporating the
geometry of the action space $\mc{A}$.  In this section, we provide an
alternative instantiation of the abstract form of
Algorithm~\ref{alg:imitation} by using Wasserstein distances as the notion of
discrepancy $\disc{\cdot}{\cdot}{s}$.  Our previous theoretical development
for KL divergences has direct analogues in this setting, which we now briefly
outline.

Given a metric $d(\cdot, \cdot)$ on $\actionset$, the Wasserstein distance
between two distributions $q^1$ and $q^2$ on $\mc{A}$ is defined by the
optimal transport problem
\begin{equation*}
  \dwass{q^1}{q^2} = \inf_{\eta \in L(q^1, q^2)}
  \E_{\eta}[d(A, A')]
\end{equation*}
where $\eta(q^1, q^2)$ denotes the collection of all probabilities on
$\actionset \times \actionset$ with marginals $q^1$ and $q^2$ (i.e.,
couplings).  Intuitively, $\dwass{q^1}{q^2}$ measures how much cost $d(A, A')$
is incurred by moving mass away from $A \sim q^1$ to $A' \sim q^2$ in an
optimal fashion\footnote{For a discrete action space, $\dwass{\cdot}{\cdot}$
  can be defined with any symmetric matrix $d(a_i, a_j)$ satisfying
  $d(a_i, a_j) \ge 0$ with $0$ iff $a_i = a_j$, and
  $d(a_i, a_j) \le d(a_i, a_k) + d(a_k, a_j)$ for any
  $a_i, a_j, a_k \in \actionset$.}. Wasserstein distances encode the geometry
of the underlying space $\actionset$ via the distance $d$. Unlike the KL
divergence $\dkl{q^1}{q^2}$ that take value $\infty$ whenever $q^1$ has
support not contained in $q^2$, Wasserstein distance allows imitation policies
to have different support than the Thompson sampling policy, which is more
appropriate in continuous action spaces. To simplify notation, for two
policies $\pi^1$ and $\pi^2$, we let
\begin{equation*}
  \dwasspolicy{\pi^1}{\pi^2}{S} \defeq
  \dwass{\pi^1(\cdot \mid S)}{\pi^2(\cdot \mid S)}.
\end{equation*}

When Algorithm~\ref{alg:imitation} is instantiated with the Wasserstein
distance as its notion of discrepancy
$\disc{\cdot}{\cdot}{S} \defeq \dwasspolicy{\cdot}{\cdot}{S}$, the imitation
learning problem~\eqref{eqn:imitation} becomes
\begin{equation}
  \label{eqn:imitation-wass}
  \minimize_{m \in \mc{M}} \E_{S \sim \P_S}
  \left[\dwasspolicy{\bar{\pi}}{\pi^m}{S}\right].
\end{equation}
To solve the above stochastic optimization problem, we can again use
stochastic gradient descent methods, where the stochastic gradient
$\nabla_m \dwasspolicy{\bar{\pi}_{\prev(t)}}{\pi^m}{S}$ can be computed by solving
an optimal transport problem. From Kantorovich-Rubinstein duality (see, for
example,~\cite{Villani09}), we have
\begin{align}
  & \dwasspolicy{\bar{\pi}_{\prev(t)}}{\pi^m}{s} \nonumber \\
  & = \sup_{g: \actionset \to \R}
    \left\{
    \E_{\bar{A} \sim \bar{\pi}(\cdot\mid s)} g(\bar{A})
    - \E_{A \sim \pi^m(\cdot\mid s)} g(A):
    ~g(a) - g(a') \le d(a, a')~\mbox{for all}~a, a' \in \actionset \right\},
    \label{eqn:kantorovich}
\end{align}
where $d(\cdot, \cdot)$ is the metric on $\actionset$ used to define
$\dwass{\cdot}{\cdot}$. For discrete action spaces, the maximization
problem~\eqref{eqn:kantorovich} is a linear program with $O(|\actionset|)$
variables and constraints; for continuous action spaces, we can solve the
problem over empirical distributions to approximate the optimal transport
problem. We refer the interested reader to~\citet{PeyreCu19} for a
comprehensive introduction to computational methods for solving optimal
transport problems.

Letting $g\opt$ denote the optimal solution to the dual
problem~\eqref{eqn:kantorovich}, the envelope theorem (or Danskin's
theorem;~\citet[Theorem 4.13]{BonnansSh00}) implies that under simple
regularity conditions
\begin{equation*}
  \nabla_m \dwasspolicy{\bar{\pi}_{\prev(t)}}{\pi^m}{s}
  = -\nabla_m\E_{A \sim \pi^m(\cdot\mid s)}[g\opt(A)].
\end{equation*}
Assuming that an appropriate change of gradient and expectation is
justified, we can use the policy gradient trick to arrive at
\begin{equation*}
  -\nabla_m\E_{A \sim \pi^m(\cdot\mid s)}[g\opt(A)]
  = -\E_{A \sim \pi^m(\cdot\mid s)}[g\opt(A) \nabla_m \log \pi^m(A \mid s)].
\end{equation*}
We conclude that for $A \sim \pi^m(\cdot \mid S_{i})$,
\begin{equation}
  \label{eqn:sg-wass}
  -g\opt(A) \nabla_m \log \pi^m(A \mid S_{i})
\end{equation}
is a stochastic gradient for the imitation
problem~\eqref{eqn:imitation-wass}. As before, we can get lower variance
estimates of the gradient by averaging the above estimator over many actions
$A \sim \pi^m(\cdot\mid S_{i})$. Using these stochastic
gradients~\eqref{eqn:sg-wass}, we can solve the imitation
problem~\eqref{eqn:imitation-wass} efficiently.

We now show that the resulting imitation policy admits a regret decomposition
similar to Lemma~\ref{lemma:imitation-controls-decomposition} for KL
divergences.  As a direct consequence of this decomposition, the regret bounds
in Section~\ref{section:regret-bound} have their natural analogues with
Wasserstein distances replacing KL divergences as the notion of discrepancy,
though we omit them for brevity.
\begin{lemma}
  \label{lemma:imitation-controls-decomposition-wass}
  Let $U_t(\cdot; H_{\prev(t)}, S_t): \actionset \to \R$ be any upper
  confidence bound sequence that is measurable with respect to
  $H_{\prev(t)}, S_t, A_t$.  If there is a $L_{\theta} > 0$ satisfying almost surely
    \begin{align}
    \label{eqn:ftn-lip}
      |f_{\theta}(a, s) -  f_{\theta}(a', s)| \le L_{\theta} d(a, a')
      ~\mbox{for all}~s \in \mc{S}, a, a' \in \actionset,
  \end{align}
  \begin{align}
    \bayesregret{T}{\{\pi_{\prev(t)}\}_{t \in \N}}
    \le & \sum_{t=1}^T \E[f_{\theta}(A_t\opt, S_t) - U_t(A_t\opt; H_{\prev(t)}, S_t)]
          + \sum_{t=1}^T \E[U_t(\bar{A}_t; H_{\prev(t)}, S_t) - f_{\theta}(\bar{A}_t, S_t)]
          \nonumber \\
        & +  \sum_{t=1}^T \E\left[ L_{\theta}
          \dwasspolicy{\bar{\pi}_{\prev(t)}}{\pi_{\prev(t)}}{S_t}\right].
          \label{eqn:regret-decomposition-wass}
  \end{align}
  where $\dwasspolicy{\cdot}{\cdot}{\cdot}$ is the Wasserstein distance
  defined with the metric $d$ in the condition~\eqref{eqn:ftn-lip}.
\end{lemma}


\paragraph{Proof} The proof mirrors that of
Lemma~\ref{lemma:imitation-controls-decomposition}. By the Kantorovich dual
representation~\eqref{eqn:kantorovich}, we have
\begin{align*}
  \E[f_{\theta}(\bar{A}_t, S_t) - f_{\theta}(A_t, S_t) \mid \theta, H_{\prev(t)}, S_t]
  \le L_{\theta} \dwasspolicy{\bar{\pi}_{\prev(t)}}{\pi_{\prev(t)}}{S_t}.
\end{align*}
Here, we have again used that $\bar{A}_t \mid H_{\prev(t)}, S_t$ and
$A_t \mid H_{\prev(t)}, S_t$ are independent of all else. Applying this bound
in the decomposition~\eqref{eqn:interim-decomposition}, and taking expectation
over $(H_{\prev(t)}, S_t)$ on both sides and summing
$t = 1, \ldots, T$, we get the desired bound. $\diamond$\\



%% file: proof-decomposition.tex
\section{Proof of Lemma~\ref{lemma:imitation-controls-decomposition}}
\label{section:proof-imitation-controls-decomposition}

Conditional on $(H_{\prev(t)}, S_t)$, $\bar{A}_t$ has the same distribution as
$A_t\opt$. Since $U_t(a; H_{\prev(t)}, S_t)$ is a deterministic function
conditional on $(H_{\prev(t)}, S_t)$, we have
\begin{equation*}
  \E[U_t(\bar{A}_t; H_{\prev(t)}, S_t) \mid H_{\prev(t)}, S_t] = \E[U_t(A_t\opt; H_{\prev(t)}, S_t)
  \mid H_{\prev(t)}, S_t].
\end{equation*}
We can rewrite the (conditional) instantenous regret as
\begin{align}
  & \E[f_{\theta}(A_t\opt, S_t) - f_{\theta}(A_t, S_t) \mid H_{\prev(t)}, S_t]  \nonumber \\
  & = \E[f_{\theta}(A_t\opt, S_t) - U_t(A_t\opt; H_{\prev(t)}, S_t) \mid H_{\prev(t)}, S_t]
    + \E[U_t(\bar{A}_t; H_{\prev(t)}, S_t) - f_{\theta}(A_t, S_t) \mid H_{\prev(t)}, S_t] 
    \nonumber \\
  & = \E[f_{\theta}(A_t\opt, S_t) - U_t(A_t\opt; H_{\prev(t)}, S_t) \mid H_{\prev(t)}, S_t] 
   + \E[U_t(\bar{A}_t; H_{\prev(t)}, S_t) -   f_{\theta}(\bar{A}_t, S_t) \mid H_{\prev(t)}, S_t]  \nonumber \\
  & \qquad  + \E[f_{\theta}(\bar{A}_t, S_t) - f_{\theta}(A_t, S_t) \mid H_{\prev(t)}, S_t].
    \label{eqn:interim-decomposition}
\end{align}

We proceed by bounding the gap
\begin{equation}
  \label{eqn:ftn-diff}
  \E[f_{\theta}(\bar{A}_t, S_t) - f_{\theta}(A_t, S_t) \mid \theta, H_{\prev(t)}, S_t]
\end{equation}
using the KL divergence between $\bar{\pi}_{\prev(t)}$ and $\pi_{\prev(t)}$.
Recall Pinsker's inequality~\citep{Tsybakov09}
\begin{equation*}
  \tvnorm{P - Q} \defeq  \sup_{g: \actionset \to [-1, 1]} | \E_{P}[g(A)] - \E_{Q}[g(A)] | 
  \le \sqrt{\half \dkl{P}{Q}}.
\end{equation*}
From the hypothesis, Pinsker's inequality implies
\begin{align*}
  \E[f_{\theta}(\bar{A}_t, S_t) - f_{\theta}(A_t, S_t) \mid \theta, H_{\prev(t)}, S_t]
  & \le \sup_{a \in \mc{A}} |f_{\theta}(a, S_t)| \tvnorm{\bar{\pi}_{\prev(t)}(\cdot \mid S_t)
    - \pi_{\prev(t)}(\cdot \mid S_t)} \\
  & \le \sup_{a \in \mc{A}} |f_{\theta}(a, S_t)| \sqrt{\half \dklpolicy{\bar{\pi}_{\prev(t)}}{\pi_{\prev(t)}}{S_t}}.
\end{align*}
Here, we have used that $\bar{A}_t \mid H_{\prev(t)}, S_t$ and
$A_t \mid H_{\prev(t)}, S_t$ are independent of all else.

Applying this bound in the decomposition~\eqref{eqn:interim-decomposition},
and taking expectation over $(H_{\prev(t)}, S_t)$ on both sides and summing
$t = 1, \ldots, T$, we get
\begin{align*}
  \bayesregret{T}{\pi}
  & \le \sum_{t=1}^T \E[f_{\theta}(A_t\opt, S_t) - U_t(A_t\opt; H_{\prev(t)}, S_t)] 
    + \sum_{t=1}^T \E[U_t(\bar{A}_t; H_{\prev(t)}, S_t) - f_{\theta}(\bar{A}_t, S_t)]
    \nonumber \\
  & \qquad + \sum_{t=1}^T 
    \E\left[ \sup_{a \in \mc{A}} |f_{\theta}(a, S_t)|
    \sqrt{\half\dklpolicy{\bar{\pi}_{\prev(t)}}{\pi_{\prev(t)}}{S_t}} \right].
\end{align*}
Applying Cauchy-Schwarz inequality and noting that
$\sqrt{\E[\sup_{a \in \mc{A}} f_{\theta}(a, S_t)^2]} \le L$, we obtain the
final decomposition.
$\diamond$ \\


%% file: proof-regret.tex
\section{Proof of regret bounds}
\label{section:proof-regret-bounds}

\subsection{Proof of Theorem~\ref{theorem:gp}}
\label{section:proof-gp}

In what follows, we abuse notation and let $C$ be a universal constant that
changes line by line. We build on the batch UCB regret bound due
to~\citet{DesautelsKrBu14}, defining the (batch) upper confidence bound
\begin{equation*}
  U_t(a; H_{\prev(t)}, s) \defeq
  \mu_{\prev(t)}(a, s) + \sqrt{\beta_t} \sigma_{\prev(t)}(a, s)
\end{equation*}
with $\beta_t = 2 \log((t^4 rd)^d t^2)$.  
From Borel-TIS inequality (e.g., see~\citep{AdlerTa09}), we have
\begin{equation*}
  L^2 = \E\left[ \sup_{a \in \mc{A}, s \in \mc{S}} f_{\theta}(a, s)^2\right]  < \infty.
\end{equation*}

We bound the first two terms in the regret
decomposition~\eqref{eqn:regret-decomposition}, starting with the second term
\begin{align*}
  \sum_{t=1}^T \E[U_t(\bar{A}_t; H_{\prev(t)}, S_t) - f_{\theta}(\bar{A}_t, S_t)]
  = \sum_{t=1}^T \sqrt{\beta_t} \E[\sigma_{\prev(t)}(\bar{A}_t, S_t)].
\end{align*}
First, note that since $|\sigma_{\prev(t)}(a, s)| \le \sqrt{c_2}$, Pinsker's
inequality gives
\begin{align*}
  |\E[\sigma_{\prev(t)}(\bar{A}_t, S_t) \mid H_{\prev(t)}, S_t]
  -  \E[\sigma_{\prev(t)}(A_t, S_t) \mid H_{\prev(t)}, S_t]|
  & \le \sqrt{c_2} \tvnorm{\bar{\pi}_{\prev(t)}(\cdot \mid S_t)
  - \pi_{\prev(t)}(\cdot \mid S_t)} \\
  & \le 
  \sqrt{\frac{c_2}{2} \dklpolicy{\bar{\pi}_{\prev(t)}}{\pi_{\prev(t)}}{S_t}}.
\end{align*}
We arrive at the interim bound
\begin{align}
  & \sum_{t=1}^T \E[U_t(\bar{A}_t; H_{\prev(t)}, S_t) - f_{\theta}(\bar{A}_t, S_t)] \nonumber \\
  & \le
  \sqrt{\beta_T} \sum_{t=1}^T  \E[\sigma_{\prev(t)}(A_t, S_t)]
  + \sqrt{\frac{\beta_T c_2}{2}}~
    \sum_{t=1}^T \sqrt{\E[\dklpolicy{\bar{\pi}_{\prev(t)}}{\pi_{\prev(t)}}{S_t}]}.
      \label{eqn:interim-bound}
\end{align}
By an elementary calculation (e.g., see~\citet[Proposition
1]{DesautelsKrBu14}), we have
\begin{equation*}
  \frac{\sigma_{\prev(t)}(a, s)}{\sigma_{t}(a, s)}
    = \exp \left(
      I\left( f(s, a), \{R_s\}_{s=\prev(t)}^{t-1} \mid \vec{R}_\prev(t)\right)
    \right)
    \le  \sqrt{\eta_{\batch}},
\end{equation*}
where the last line follows from Assumption~\ref{assumption:batch-info}. Next,
we use the following lemma due to~\citet{SrinivasKrKaSe12}.
\begin{lemma}[{\citet[Lemma 5.3]{SrinivasKrKaSe12}}]
  \label{lemma:max-info}
  For any sequence of $A_t$ and $S_t$, 
  \begin{equation*}
    \E \left( \sum_{t=1}^T \sigma_t^2(A_t, S_t) \right)^{\half}
    \le \sqrt{\frac{2 \gamma_T}{\log (1+ \sigma^{-2})}}
  \end{equation*}
\end{lemma}
\noindent Using the preceding two bounds, RHS of the
inequality~\eqref{eqn:interim-bound} can be further bounded by
\begin{align}
  \label{eqn:second-term-bound}
  & \sum_{t=1}^T \E[U_t(\bar{A}_t; H_{\prev(t)}, S_t) - f_{\theta}(\bar{A}_t, S_t)]
    \nonumber \\
  & \le  \sqrt{T \eta_{\batch} \beta_T}  \E\left[
    \left(\sum_{t=1}^T  \sigma_t^2(A_t, S_t)\right)^{\half} \right]
   + \sqrt{\frac{\beta_T c_2}{2}}~
    \sum_{t=1}^T \sqrt{\E[\dklpolicy{\bar{\pi}_{\prev(t)}}{\pi_{\prev(t)}}{S_t}]}
    \nonumber \\
  & \le  \sqrt{\frac{2 T\eta_{\batch} \gamma_T  \beta_T}{\log (1+ \sigma^{-2})}}
    + \sqrt{\frac{\beta_T c_2}{2}}~
    \sum_{t=1}^T \sqrt{\E[\dklpolicy{\bar{\pi}_{\prev(t)}}{\pi_{\prev(t)}}{S_t}]}.
\end{align}

We now bound the first term in the decomposition~\eqref{eqn:regret-decomposition}.
Let $\mc{A}_t$ be a
$(1/t^4)$-cover of $\mc{A}$, so that for any $a \in \mc{A}$, there exists
$[a]_t \in \mc{A}_t$ such that $\norm{a - [a]_t}_1 \le 1/t^4$.
\begin{align*}
  & \sum_{t=1}^T \E[f_{\theta}(A_t\opt, S_t) - U_t(A_t\opt; H_{\prev(t)}, S_t)]
  =  \underbrace{\sum_{t=1}^T
      \E[f_{\theta}(A_t\opt, S_t) - f_{\theta}([A_t\opt]_t, S_t)]}_{(a)} \\
      & \qquad + \underbrace{\sum_{t=1}^T \E[f_{\theta}([A_t\opt]_t, S_t)
      - U_t([A_t\opt]_t; H_{\prev(t)}, S_t)]}_{(b)} 
     + \underbrace{\sum_{t=1}^T \E[ U_t([A_t\opt]_t; H_{\prev(t)}, S_t)
      - U_t(A_t\opt; H_{\prev(t)}, S_t)]}_{(c)}.
\end{align*}
Using the definition of $L_f$, the first term $(a)$ in the above equality is
bounded by
\begin{equation*}
  \sum_{t=1}^T \E[f_{\theta}(A_t\opt, S_t) - f_{\theta}([A_t\opt]_t, S_t)]
  \le \E[L_f] \sum_{t=1}^T \norm{A_t\opt - [A_t\opt]_t}_1
  \le \E[L_f] \sum_{t=1}^\infty \frac{1}{t^4} \le C \E[L_f]
\end{equation*}
where we used the fact that $\mc{A}_t$ is a $1/t^4$-cover of $\mc{A}$.

To bound the second term $(b)$, use
$f_{\theta}(a, s) \mid H_{\prev(t)} \sim N(\mu_{\prev(t)}(a, s),
\sigma_{\prev(t)}^2(a, s))$
\begin{align}
  \E[f_{\theta}(a, s) - U_t(a; H_{\prev(t)}, s) \mid H_{\prev(t)}]
  & \le   \E[\hinge{f_{\theta}(a, s) - U_t(a; H_{\prev(t)}, s)} \mid H_{\prev(t)}]
  \nonumber \\
  & = \frac{\sigma_{\prev(t)}(a, s)}{\sqrt{2\pi}} e^{-\frac{\beta_t}{2}}
  \le \frac{c_2}{\sqrt{2\pi} t^2 |\mc{A}_t|},   \label{eqn:normal-integral}
\end{align}
where we used $2 \log (|\mc{A}_t| t^2) \le \beta_t$ since
$|\mc{A}_t| \le (t^4 r d)^d$.  Hence, we obtain the bound
\begin{align*}
  \sum_{t=1}^T \E[f_{\theta}([A_t\opt]_t, S_t) - U_t([A_t\opt]_t; H_{\prev(t)}, S_t)]
  \le \sum_{t=1}^T \sum_{a \in \mc{A}_t}
  \E[f_{\theta}(a, S_t) - U_t(a; H_{\prev(t)}, S_t)]
  \le \sum_{t=1}^\infty  \frac{c_2}{\sqrt{2\pi} t^2}
  \le C c_2
\end{align*}
where we used the independence of $S_t$ and $H_{\prev(t)}$, and the
bound~\eqref{eqn:normal-integral}.

To bound the third term $(c)$, we show the claim
\begin{align}
  & |U_t(a; H_{\prev(t)}, s) - U_t(a'; H_{\prev(t)}, s)|
   \le \E[L_f \mid H_{\prev(t)}] \norm{a-a'}_1   \label{eqn:ucb-lipschitz} \\
  & \qquad + \sqrt{\beta_t}  \left(2 \E\left[L_f
    \left(\sup_{a\in \mc{A}, s \in \mc{S}} \mu(a, s)^2
    + \sup_{a \in \mc{A}, s \in \mc{S}}
    f_{\theta}(a, s)^2 \right) \mid H_{\prev(t)} \right]\right)^{\half} \norm{a-a'}_1^{\half}.
    \nonumber 
\end{align}
From the above claimed bound, it follows that
\begin{align*}
  \sum_{t=1}^T \E[ U_t([A_t\opt]_t; H_{\prev(t)}, S_t)
  - U_t(A_t\opt; H_{\prev(t)}, S_t)]
  & \le \sum_{t=1}^T \frac{\E[L_f]}{t^4}
  + \sum_{t=1}^T \sqrt{2 \beta_t}
  \frac{c_1 \sqrt{\E[L_f] } + c_3 \sqrt{\E[L_f^2]}}{t^2} \\
  & \le C \E[L_f] + C d \log (rd) \left(
    c_1 \sqrt{\E[L_f] } + c_3 \sqrt{\E[L_f^2]}\right).
\end{align*}

To show the bound~\eqref{eqn:ucb-lipschitz}, first note
$a \mapsto \E[f_{\theta}(a, s) \mid H_{\prev(t)}]$ and
$a \mapsto \E[f_{\theta}(a, s)^2 \mid H_{\prev(t)}]$ are
$\E[L_f \mid H_{\prev(t)}]$- and
$\E[2 L_f \sup_{a \in \mc{A}, s \in \mc{S}} |f_{\theta}(a, s)| \mid
H_{\prev(t)}]$- Lipschitz respectively, for all $s \in \mc{S}$. Hence,
$a \mapsto \sigma_{\prev(t)}^2(a, s)$ is
$\E[2 L_f (c_1^2 + \sup_{a \in \mc{A}, s \in \mc{S}} |f_{\theta}(a, s)|^2)
\mid H_{\prev(t)}]$-Lipschitz.  Noting that
\begin{align*}
  |\sigma_{\prev(t)}(a, s) - \sigma_{\prev(t)}(a', s) |
  = \left|
  \frac{\sigma_{\prev(t)}^2(a, s) - \sigma_{\prev(t)}^2(a', s)}{\sigma_{\prev(t)}(a, s) + \sigma_{\prev(t)}(a', s)}
  \right|
  \le \frac{1}{c} |\sigma_{\prev(t)}^2(a, s) - \sigma_{\prev(t)}^2(a', s)|
    + c
\end{align*}
for any $c> 0$, taking the infimum over $c>0$ on the right hand side yields
\begin{align*}
  |\sigma_{\prev(t)}(a, s) - \sigma_{\prev(t)}(a', s) |
  & \le \sqrt{2 |\sigma_{\prev(t)}^2(a, s) - \sigma_{\prev(t)}^2(a', s)|} \\
  & \le \left(2 \E\left[L_f
    \left(c_1^2 + \sup_{a \in \mc{A}, s \in \mc{S}}
    f_{\theta}(a, s)^2 \right) \mid H_{\prev(t)} \right]\right)^{\half} \norm{a-a'}_1^{\half}
\end{align*}
which shows the bound~\eqref{eqn:ucb-lipschitz}.

Collecting these bounds, we have shown that
\begin{align}
  \label{eqn:first-term-bound}
  \sum_{t=1}^T \E[f_{\theta}(A_t\opt, S_t) - U_t(A_t\opt; H_{\prev(t)}, S_t)]
  \le C \E[L_f] + C c_2  + C d \log (rd) \left(
    c_1 \sqrt{\E[L_f] } + c_3 \sqrt{\E[L_f^2]}\right).
\end{align}
Combining this with the bound~\eqref{eqn:second-term-bound}, we obtain our
result.


%% file: proof-generalization.tex
\section{Proof of generalization results}

\subsection{Proof of Theorem~\ref{theorem:generalization}}
\label{section:proof-generalization}

We abuse notation and use $C > 0$ to denote a numerical constant that changes
value line to line.  We use the following concentration guarantee using
localized Rademacher averages.
\begin{lemma}[{\citet[Theorem 3.3]{BartlettBoMe05}}]
  \label{lemma:local-rad}
  For a class of functions $\mc{G}$ with range $[0, \lbound]$, let $r_n\opt$
  be the unique positive fixed point of the sub-root function $\psi_n$
  satisfying the bound~\eqref{eqn:sub-root}. Then, for i.i.d. observations
  $\xi \simiid \P$, there is a numerical constant $C>0$ such that 
\begin{equation*}
  \E[g]
  \le \left(1+\frac{1}{\eta}\right) \frac{1}{n} \sum_{i=1}^n g(\xi_i)
  + C(1 + \eta)\left( \frac{1}{\lbound} r_n\opt + \frac{\lbound z}{n}\right)
  + \frac{C\lbound z}{n}
  ~~ \mbox{for~all~} g \in \mc{G} ~ \mbox{and}~ \eta \ge 0
\end{equation*}
with probability at least $1 - e^{-z}$.
\end{lemma}
\noindent Notice that by Jensen inequality, we have
\begin{align*}
  \E_{\bar{A} \sim \bar{\pi}(\cdot \mid S_i)}\left[
  \log \frac{\bar{\pi}(\bar{A} \mid S_i)}{\pi^{\what{m}}(\bar{A} \mid S_i)}
  \right] \ge 0 ~~\mbox{almost surely}.
\end{align*}
Applying Lemma~\ref{lemma:local-rad} with the function class
$\mc{G}_1$ and $\eta = 1/2$, we have
\begin{align}
  \label{eqn:loc-rad-bound}
  \E\left[\dklpolicy{\bar{\pi}}{\pi^{\what{m}}}{S}\right]
  \le \frac{3}{2} \frac{1}{N} \sum_{i=1}^N
  \E_{\bar{A} \sim \bar{\pi}(\cdot \mid S_i)}\left[
  \log \frac{\bar{\pi}(\bar{A} \mid S_i)}{\pi^{\what{m}}(\bar{A} \mid S_i)}
  \right]
  + \frac{C}{\lbound}r_N\opt + \frac{C\lbound z}{N}.
\end{align}

In the rest of the proof, we bound the interim (uniform) approximation error
\begin{equation*}
  Z_{N, N_a} \defeq 
  \sup_{m \in \mc{M}}\left\{
    \frac{1}{N} \sum_{i=1}^N  \left(
      \E_{\bar{A} \sim \bar{\pi}(\cdot \mid S_i)}
      \left[\log \frac{\bar{\pi}(\bar{A} \mid S_i)}{\pi^{m}(\bar{A} \mid S_i)} \right]
      - \frac{1}{N_a} \sum_{j=1}^{N_a}
      \log \frac{\bar{\pi}(\bar{A}_{ij} \mid S_i)}
      {\pi^{m}(\bar{A}_{ij} \mid S_i)}
    \right)
  \right\}.
\end{equation*}
This is indeed sufficient for our purposes since the bound~\eqref{eqn:loc-rad-bound} implies
\begin{align*}
  \E\left[\dklpolicy{\bar{\pi}}{\pi^{\what{m}}}{S}\right]
  \le \frac{3}{2} \frac{1}{N} \sum_{i=1}^N  
  \frac{1}{N_a} \sum_{j=1}^{N_a}
  \log \frac{\bar{\pi}(\bar{A}_{ij} \mid S_i)}{\pi^{\what{m}}(\bar{A}_{ij} \mid S_i)}
  + C \left( Z_{N, N_a}
  + \frac{1}{\lbound}r_N\opt + \frac{\lbound z}{N}\right).
\end{align*}
By the definition~\eqref{eqn:real-empirical-il} of the empirical solution
$\what{m}$ and by virtue of having a well-specified model class $\mc{M}$, the first
term in the preceding bound is nonpositive
\begin{align*}
  \frac{1}{N} \sum_{i=1}^N  
  \frac{1}{N_a} \sum_{j=1}^{N_a}
  \log \frac{\bar{\pi}(\bar{A}_{ij} \mid S_i)}{\pi^{\what{m}}(\bar{A}_{ij} \mid S_i)}
  \le   \frac{1}{N} \sum_{i=1}^N  
  \frac{1}{N_a} \sum_{j=1}^{N_a}
  \log \frac{\bar{\pi}(\bar{A}_{ij} \mid S_i)}{\bar{\pi}^{m\opt}(\bar{A}_{ij} \mid S_i)} = 0.
\end{align*}

Consider the Doob martingale $M_0 = \E[Z_{N, N_a}]$, and
\begin{equation*}
  M_k \defeq \E[Z_{N, N_a} \mid S_1, \ldots, S_k]~~~\mbox{for}~~~ 1\le k \le N,
\end{equation*}
a martingale adapted to the filtration
$\mc{F}_k \defeq \sigma(S_1, \ldots, S_k)$. Denote the martingale difference
sequence $D_k = M_k - M_{k-1}$ for $k \ge 1$. Let $S_k'$ be an
independent copy of $S_k$ that is independent of all $S_i$, and let
$\bar{A}_{kj}' \sim \bar{\pi}(\cdot \mid S_k')$ independent of
everything other than $S_k'$. We can write
\begin{align*}
  D_k 
  & = \E\left[
    \sup_{m \in \mc{M}}\left\{
    \frac{1}{N} \sum_{i=1}^N   \left(
    \E_{\bar{A} \sim \bar{\pi}(\cdot \mid S_i)}
    \left[\log \frac{\bar{\pi}(\bar{A} \mid S_i)}{\pi^{m}(\bar{A}
    \mid S_i)} \right]
    - \frac{1}{N_a} \sum_{j=1}^{N_a}
    \log \frac{\bar{\pi}(\bar{A}_{ij} \mid S_i)}
    {\pi^{m}(\bar{A}_{ij} \mid S_i)}
    \right)
    \right\}  \mid S_1, \ldots, S_k\right]\\
  & \qquad - \E\Bigg[
    \sup_{m \in \mc{M}}\Bigg\{
    \frac{1}{N} \sum_{i\neq k}  \left(
    \E_{\bar{A} \sim \bar{\pi}(\cdot \mid S_i)}
    \left[\log \frac{\bar{\pi}(\bar{A} \mid S_i)}{\pi^{m}(\bar{A}
    \mid S_i)} \right]
    - \frac{1}{N_a} \sum_{j=1}^{N_a}
    \log \frac{\bar{\pi}(\bar{A}_{ij} \mid S_i)}
    {\pi^{m}(\bar{A}_{ij} \mid S_i)}
    \right) \\
  &  \qquad \qquad \qquad
    + \frac{1}{N}  \left(
    \E_{\bar{A}' \sim \bar{\pi}(\cdot \mid S_k')}
    \left[\log \frac{\bar{\pi}(\bar{A}' \mid S_k')}
    {\pi^{m}(\bar{A}'\mid S_k')} \right]
    - \frac{1}{N_a} \sum_{j=1}^{N_a}
    \log \frac{\bar{\pi}(\bar{A}_{ij}' \mid S_k')}
    {\pi^{m}(\bar{A}'_{ij} \mid S_k')}
    \right) \Bigg\}   \Bigg| S_1, \ldots, S_{k}
    \Bigg].
\end{align*}
Independence of $S_i$'s yields
\begin{align*}
  |D_k|
  & \le \frac{1}{N} \E\Bigg[\sup_{m \in \mc{M}}
    \Bigg|
    \frac{1}{N_a} \sum_{j=1}^{N_a} 
     \E_{\bar{A} \sim \bar{\pi}(\cdot \mid S_k)}
    \left[\log \frac{\bar{\pi}(\bar{A} \mid S_k)}
    {\pi^{m}(\bar{A}\mid S_k)} \right]
     - \log \frac{\bar{\pi}(\bar{A}_{ij} \mid S_k)}
    {\pi^{m}(\bar{A}_{ij} \mid S_k)} \\
  & \qquad \qquad \qquad  \qquad \qquad \qquad 
    -  \E_{\bar{A}' \sim \bar{\pi}(\cdot \mid S_k')}
    \left[\log \frac{\bar{\pi}(\bar{A}' \mid S_k')}
    {\pi^{m}(\bar{A}'\mid S_k')} \right]
     + \log \frac{\bar{\pi}(\bar{A}_{ij}' \mid S_k')}
    {\pi^{m}(\bar{A}'_{ij} \mid S_k')}
    \Bigg|\mid S_k \Bigg] \\
  & \le \frac{2}{N} \sup_{ s \in \mc{S}}
  \E_{\bar{A}_{j} \simiid \bar{\pi}(\cdot \mid s)}
  \left[\sup_{m \in \mc{M}}
  \left|
  \frac{1}{N_a} \sum_{j=1}^{N_a}
  \E_{\bar{A} \sim \bar{\pi}(\cdot \mid s)}
    \left[\log \frac{\bar{\pi}(\bar{A} \mid s)}
    {\pi^{m}(\bar{A}\mid s)} \right]
     - \log \frac{\bar{\pi}(\bar{A}_{j} \mid s)}
    {\pi^{m}(\bar{A}_{j} \mid s)}
  \right| \right].
\end{align*}

Next, we use a standard symmetrization result to bound the last display;
see, for example, Chapter 2.3,~\citet{VanDerVaartWe96} for a comprehensive
treatment.
\begin{lemma}
  \label{lemma:symmetrization}
  If $\xi_i \simiid P$, we have
  \begin{equation*}
    \E\left[\sup_{g \in \mc{G}} \left|
        \frac{1}{n} \sum_{i=1}^n
        (g(\xi_i) - \E[g(\xi)])
    \right|
  \right]
  \le 4 \E[\radcomp_n(\mc{G})]
  \end{equation*}
\end{lemma}
\noindent Applying Lemma~\ref{lemma:symmetrization} to the bound on $|D_k|$.
we conclude
$|D_k| \le \frac{8}{N} \sup_{s \in \mc{S}}
\E_{\bar{A}_j \simiid \bar{\pi}(\cdot \mid s)}[ \radcomp_{N_a} (\mc{G}_2'(s))]$,
where $\mc{G}_2'(s)$ is the function class
\begin{align*}
  \mc{G}_2'(s)   \defeq \left\{
  a \mapsto \log \frac{\bar{\pi}(a \mid s)}{\pi^m(a \mid s)}:
  m \in \mc{M} \right\}.
\end{align*}
Note that $\radcomp_{N_a} (\mc{G}_2'(s)) = \radcomp_{N_a} (\mc{G}_2(s))$.
Then, Azuma-Hoeffding bound (Corollary 2.20,~\citet{Wainwright19}) yields
\begin{align*}
  Z_{N, N_a} \le \E[Z_{N, N_a}]  + \sqrt{\frac{32z}{N}} \sup_{s \in \mc{S}}\E_{\bar{A}_j \simiid \bar{\pi}(\cdot \mid s)}[ \radcomp_{N_a} (\mc{G}_2(s))]
\end{align*}
with probability at least $1-e^{-z}$.

It now remains to bound $\E[Z_{N, N_a}]$, for which we use a symmetrization
argument. Although $(S_i, \bar{A}_{ij})$ are not i.i.d., a standard
argument still applies, which we outline for completeness. Denoting by
$(S_i', \bar{A}'_{ij})$ independent copies of
$(S_i, \bar{A}_{ij})$ again,  we have
\begin{align*}
  \E[Z_{N, N_a}]
  & =  \E\left[\sup_{m \in \mc{M}}
    \left|
    \frac{1}{N} \sum_{i=1}^N  \frac{1}{N_a} \sum_{j=1}^{N_a}
    \log \frac{\bar{\pi}(\bar{A}_{ij} \mid S_i)}
    {\pi^{m}(\bar{A}_{ij} \mid S_i)}
    -  \E\left[ \frac{1}{N} \sum_{i=1}^N \frac{1}{N_a} \sum_{j=1}^{N_a}
    \log \frac{\bar{\pi}(\bar{A}'_{ij} \mid S_i')}
    {\pi^{m}(\bar{A}'_{ij} \mid S_i')}\right]
    \right|\right] \\
  & \le 
    \E\left[\sup_{m \in \mc{M}}
    \left|
    \frac{1}{N} \sum_{i=1}^N  \frac{1}{N_a} \sum_{j=1}^{N_a}
    \log \frac{\bar{\pi}(\bar{A}_{ij} \mid S_i)}
    {\pi^{m}(\bar{A}_{ij} \mid S_i)}
    - \log \frac{\bar{\pi}(\bar{A}'_{ij} \mid S_i')}
    {\pi^{m}(\bar{A}'_{ij} \mid S_i')}
    \right|\right] \\
  & =  
    \E\left[\sup_{m \in \mc{M}}
    \left|
    \frac{1}{N} \sum_{i=1}^N  \frac{1}{N_a} \sum_{j=1}^{N_a}
    \epsilon_{ij} \left(
    \log \frac{\bar{\pi}(\bar{A}_{ij} \mid S_i)}
    {\pi^{m}(\bar{A}_{ij} \mid S_i)}
    - \log \frac{\bar{\pi}(\bar{A}'_{ij} \mid S_i')}
    {\pi^{m}(\bar{A}'_{ij} \mid S_i')} \right)
    \right|\right] \\
  & \le 2 \E[\radcomp_{NN_a}(\mc{G}_3)].
\end{align*}
Collecting these bounds, we obtain the desired result.

\subsection{Proof of Corollary~\ref{corollary:kernel}}
\label{section:proof-kernel}

We use the following standard result that bound the Rademacher complexity of
kernel models.  Let $k$ be a reproducing kernel on $\Xi$, and let $\B$ be the
unit ball in the RKHS $\mc{H}$.
\begin{claim}
  \label{claim:kernel}
  Let $\sup_{\xi \in \Xi} k(\xi, \xi) = B < \infty$. Then,
  $\radcomp_{n}(\B) \le \frac{B}{\sqrt{n}}$.
\end{claim}
\paragraph{Proof of Claim~\ref{claim:kernel}}
For any fixed $\xi_1,\ldots, \xi_n$,
\begin{align*}
  \radcomp_n(\B)
  & = \frac{1}{n} \E_{\varepsilon}\left[\sup_{h \in \B} 
  \< h, \sum_{i=1}^n \varepsilon_i k(\cdot, \xi_i)\> \right]
    = \frac{1}{n} \E_{\varepsilon}\left[
    \norm{\sum_{i=1}^n \varepsilon_i k(\cdot, \xi_i)}_{\mc{H}} \right] \\
  & \le \frac{1}{n} \left( \E_{\varepsilon}\left[
    \norm{\sum_{i=1}^n \varepsilon_i k(\cdot, \xi_i)}_{\mc{H}}^2 \right]\right)^{\half} 
    = \frac{1}{n} \sqrt{\sum_{i=1}^n \norm{k(\cdot, \xi_i)}_{\mc{H}}^2} \le \frac{B}{\sqrt{n}}.
    \qed
\end{align*}

Applying the claim to $\B_{\mc{A}}$ and and $\B_{\mc{S}\times\mc{A}}$, we get
\begin{align*}
  \sup_{s \in \mc{S}} \radcomp_{N_a} (\mc{G}_2(s)) \le \frac{B}{\sqrt{N_a}}~~~\mbox{and}~~~\radcomp_{NN_a}(\mc{G}_3)\le \frac{B}{\sqrt{NN_a}}.
\end{align*}

To bound $r_N\opt$, we use the following result due to~\citet{Mendelson03}.
\begin{lemma}[{\citet[Theorem 2.1]{Mendelson03}}]
  \label{lemma:mendelson-rkhs}
  If $\lambda_1 \ge 1/N$, then for all $r \ge 1/N$
  \begin{align*}
    \E\left[\radcomp_{N}\left\{ h \in \B: \E[h(S)^2] \le r \right\} \right]
    \lesssim \left(\frac{1}{N} \sum_{j = 1}^\infty
    \min\{\lambda_j, r \}\right)^\half.
  \end{align*}
\end{lemma}
\noindent Consider the case where the spectrum of $T_{k_{\mc{S}}}$ decay exponentially
\begin{align*}
  \left(\frac{1}{N} \sum_{j = 1}^\infty \min\left\{e^{-j^2},
  \frac{\sqrt{\log N}}{N} \right\}\right)^\half 
  \lesssim \left(\frac{1}{N} \sum_{j = 1}^{\sqrt{\log N}}
    \frac{\sqrt{\log N}}{N}
  + \frac{1}{N} \int_{\sqrt{\log N}}^\infty e^{-t^2} dt \right)^\half
  \lesssim \frac{\sqrt{\log N}}{N},
\end{align*}
where we use $\lesssim$ to denote inequality up to a numerical constant. We
conclude $r_{N}\opt \lesssim \lbound  \frac{\sqrt{\log N}}{N}$.  For polynomially
decaying spectrum $\lambda_j \lesssim j^{-2 \beta}$,
\begin{equation*}
  \sum_{j = 1}^\infty \min\{j^{-2\beta}, r\}
  \approx
  r^{\frac{2 \beta - 1}{2 \beta}}
  + \int_{r^{-1 / 2\beta}}^\infty t^{-2\beta} dt
  \asymp
  r^{\frac{2 \beta - 1}{2\beta}}.
\end{equation*}
Solving for the fixed point, we get
$r_N\opt  \asymp \lbound n^{-\frac{2\beta}{2\beta +1}}$.

Collecting these bounds and plugging them into Theorem~\ref{theorem:generalization}, we obtain the desired result.


%% file: experiment_details.tex
\section{Experiment Details}
\label{section:experiment_details}

\paragraph{Hyperparameters}
We use hyperparameters from \citet{RiquelmeTuSn18} as follows. The
\textsc{NeuralGreedy, NeuralLinearTS} methods use a fully-connected neural
network with two hidden layers of containing 100 rectified linear units. The
networks are multi-output, where each output corresponds for predicted reward
under each action. The networks are trained using 100 mini-batch updates at
each period to minimize the mean-squared error via RMSProp with an initial
learning rate of 0.01. The learning rate is decayed after each mini-batch
update according to an inverse time decay schedule with a decay rate of 0.55
and the learning rate is reset the initial learning rate each update
period. For \textsc{Bootstrap-NN-TS}, we use 10 replicates and train each
replicate with all observations as in \citet{RiquelmeTuSn18}.

The Bayesian linear regression models used on the last linear layer for
\textsc{NeuralLinear-TS} use the normal inverse gamma prior
$\mbox{NIG}(\mu_a=\bm 0, \alpha_a=3, \beta_a=3,
\Lambda_a=0.25I_d$). \textsc{Linear-TS} uses a
$\mbox{NIG}(\mu_a=\bm 0, \alpha_a=6, \beta_a=6, \Lambda_a=0.25I_d$) prior
distribution.

The imitation models used by the \textsc{IL} methods are fully-connected
neural networks with two hidden layers of 100 units and hyperbolic tangent
activations. The networks use a Softmax function on the outputs to predict the
probability of selecting each action. The networks are trained using 2000
mini-batch updates via RMSProp to minimize the KL-divergence between the
predicted probabilities and the approximate propensity scores of the Thompson
sampling policy $\pi^{TS}$. For each observed context $S_i$, we approximate
the propensity scores of the Thompson sampling policy $\pi^{TS}(\cdot|S_i)$
using $N_a=2048$ Monte Carlo samples:
$\hat\pi^{TS}(a|S_i) = \frac{1}{N_a}\sum_{j=1}^{N_a} \indic{A_{ij} = a}$ where
$A_{ij}\sim \pi^{TS}(\cdot|S_i)$. We use an initial learning rate of
0.001. learning rate is decayed every 100 mini-batches according to an inverse
time decay schedule with a decay rate of 0.05. In practice, the
hyperparameters of the imitation model can be optimized or adjusted at each
update period by minimizing the KL-divergence on a held-out subset of the
observed data, which may lead to better regret performance. We do not use
inverse propensity-weighting on the observations, but we suspect that may it
may further improve performance.

We normalize all numeric features to be in [0,1] and one-hot encode all
categorical features. For the Warfarin dataset, we also normalize the rewards
to be in [0,1].

\paragraph{Posterior Inference for Bayesian Linear Regression}
\label{subsection:algorithms} 
\textbf{\textsc{Linear-TS}}: For each action, We assume the data for action $a$ were generated from the linear function: $r_a = \bm{s}^T \bm{\theta}_a + \varepsilon $ where $\varepsilon \sim \mathcal N (0, \sigma_a^2)$.
\begin{equation*}
  \sigma_a^2 \sim \mbox{IG}(\alpha_a, \beta_a), ~~\bm{\theta_a}|\sigma_a^2 \sim \mathcal N (\bm{\mu}_a, \sigma_a^2\Sigma_a),
\end{equation*} where the prior distribution is given by $\mbox{NIG}(\bm \mu_a, \Lambda_a, \alpha_a, \beta_a)$ and $\Lambda_a = \Sigma_a^{-1}$ is the precision matrix. After $n_a$ observations of contexts $X_a \in \R^{n_a \times (d+1)}$ and rewards $\bm y_a \in \R^{n_a \times 1}$, we denote the joint posterior by $P(\bm \theta_a, \sigma_a^2) \sim \mbox{NIG}(\bar{\bm \mu}_a, \bar{\Lambda}_a, \bar{\alpha}_a, \bar{\beta}_a)$, where
\begin{gather*}
\bar{\Lambda} =  X_a^T X_a + \Lambda_a, ~~ \bar{\bm \mu}_a = \bar{\Lambda}_a^{-1}(\Lambda_a \bm \mu_a + X_a^T \bm y_a)\\
\bar{\alpha}_a = \alpha + \frac{n_a}{2}, ~~ \bar{\beta}_a = \beta + \frac{1}{2}(\bm y_a ^T \bm y_a + \bm \mu_a^T \Lambda_a \bm \mu_a - \bar{\bm \mu}_a^T \bar{\Lambda}_a \bar{\bm \mu}_a).
\end{gather*}

\paragraph{Additional Results}
\textbf{Warfarin - 50 Actions} Figure \ref{fig:warfarin50_regret} shows the
cumulative regret on Warfarin using 50 actions. The imitation learning methods
match the cumulative regret of the vanilla Thompson sampling methods.
\begin{figure}[t!]
    \centering
    \includegraphics[width=.4\columnwidth]{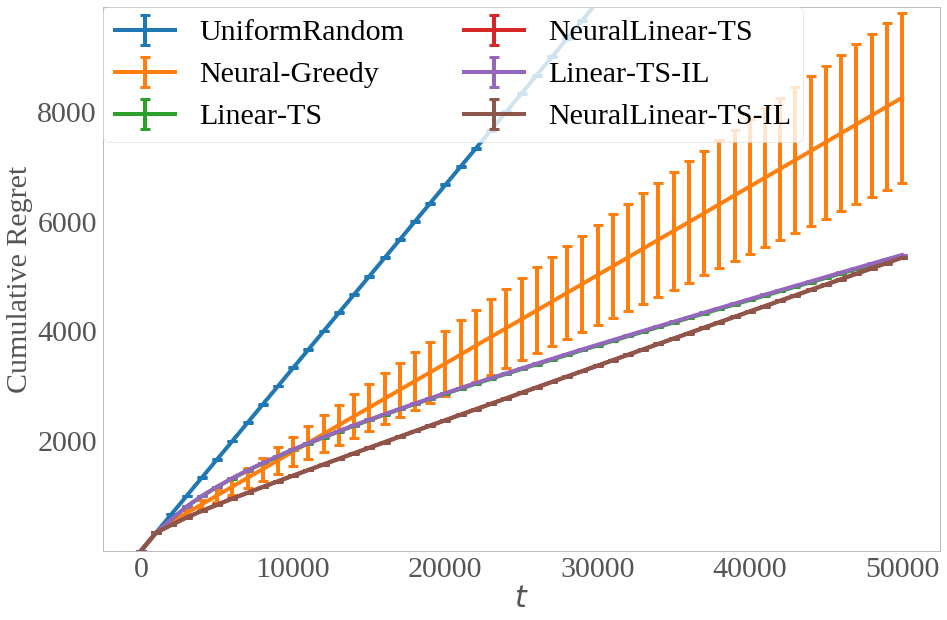}
    \caption{\label{fig:warfarin50_regret} Cumulative regret on the Warfarin problem with 50 actions}
\end{figure}


%% file: time_complexity.tex
\section{Time and Space Complexity}
\label{section:time_complexity}
\subsection{Complexity of Evaluated Methods}
Table \ref{table:complexity} shows the decision-making time complexity for the methods used in our empirical analysis. The time complexity is equivalent to the space complexity for all evaluated methods.

\textbf{\textsc{NeuralGreedy}}
The time complexity of \textsc{NeuralGreedy} is the sum of matrix-vector multiplications involved in a forward pass. 

\textbf{\textsc{Linear-TS}}
The time complexity of \textsc{Linear-TS} is dominated by sampling from the joint posterior, which requires sampling from a multivariate normal with dimension $d$. To draw a sample from the joint posterior $P(\bm \theta, \sigma)$ at decision time, we first sample the noise level $\tilde\sigma^2 \sim \text{IG}(\alpha, \beta)$ and then sample $\tilde{\bm \theta}|\tilde\sigma^2 \sim \mathcal N\big(\bm \mu, \tilde\sigma^2\Lambda^{-1}\big)$. Rather than inverting the precision matrix $\tilde\Sigma = \tilde\sigma^2 \Lambda^{-1}$, we compute root decomposition (e.g. a Cholesky decomposition) of the $d \times d$ precision matrix $\Lambda = LL^T$. The root decomposition can be computed once, with cost $O(d^3)$, after an offline batch update and cached until the next batch update. Given $L^T$, we sample directly by computing $\tilde{\bm \theta} = \bm\mu + \bm z$, where
\begin{equation}\label{eqn:sampling_linear}
\frac{1}{\tilde\sigma}L^T\bm z = \bm \zeta
\end{equation} 
and $\bm \zeta \overset{\text{iid}}{\sim} \mathcal N(0, 1)$. Since $L^T$ is upper triangular, Eqn.~\eqref{eqn:sampling_linear} can be solved using a backward substitution in quadratic time: $O(d^2)$.\footnote{The alternative approach of inverting the precision matrix to compute the covariance matrix $\Sigma = \Lambda^{-1}$, computing and caching its root decomposition $\Sigma= L_\Sigma L_\Sigma^T$, and sampling $\tilde{\bm\theta}$ as $\tilde{\bm\theta} = \bm \mu + L_\Sigma \bm \zeta,$ where $\bm \zeta \overset{\text{iid}}{\sim} \mathcal N(0, 1)$ also has a time complexity of $O(d^2)$ from the matrix-vector multiplication $L_\Sigma \bm \zeta$.}

\textbf{\textsc{NeuralLinear-TS}} 
The time complexity of \textsc{NeuralLinear-TS} is the sum of a forward pass up to the last hidden layer and sampling from a multivariate normal with dimension $h_M$, where $h_M$ is the size of the last hidden layer. 

\textbf{\textsc{Imitation Learning}}
The \textsc{IL} methods have the same time complexity as \textsc{NeuralGreedy}, ignoring the cost of sampling from multinomial with $k$ categories.

\subsection{Complexity Using Embedded Actions}
An alternative modeling approach for the non-imitation methods is to embed the action with the context as input to the reward model. 

\textbf{\textsc{NeuralGreedy}}
Using an embedded action, the time complexity for a forward pass up to the last layer is $O_{\text{last-layer}} = O\big(kd_ah_1 + k\sum_{m=1}^{M-1} h_mh_{m+1}\big)$ because the input at decision time is a $k \times d_a$ matrix where the context is embedded with each of the $k$ actions and the each context-action vector has dimension $d_a$. The time complexity of computing the output layer remains $O(kh_M)$. The space complexity remains linear in the number of parameters, but it also requires computing temporary intermediate tensors of size $k \times h_m$  for $m=1...M$: $O\big(d_ah_1 + \sum_{m=1}^{M-1} h_mh_{m+1} + \sum_{m=1}^{M}kh_m\big)$.

\textbf{\textsc{Linear-TS}}
Linear-TS with an embedded action only requires using a single sample of the parameters, which yields a complexity of to $O(d_a^2+kd_a)$ for \textsc{Linear-TS}. The space complexity is also $O(d_a^2+kd_a)$.

\textbf{\textsc{NeuralLinear-TS}}
For \textsc{NeuralLinear-TS} the time complexity of computing the outputs given the last hidden layer is $O(h_M^2 + kh_M)$, since only a single sample of $h_M$ parameters is required for computed the reward for all actions. The space complexity for \textsc{NeuralLinear-TS} the sum the space complexities of \textsc{NeuralGreedy} and \textsc{Linear-TS}.

\textbf{\textsc{Imitation Learning}}
The computatiuonal cost of the \textsc{IL} methods would be unchanged.

We choose to empirically evaluate models \emph{without} embedded actions because linear methods using embedded actions cannot model reward functions that involve non-linear interactions between the contexts and actions, whereas modeling each action independently allows for more flexibility. \citet{RiquelmeTuSn18}  find that Thompson sampling using disjoint, exact linear bayesian regressions are a strong baseline in many applications. Furthermore, \citet{RiquelmeTuSn18} observe that it is important to model the noise levels independently for each action.

\subsection{Complexity of Alternative Methods}
Alternative Thompson sampling methods including mean-field approaches, the low-rank approximations of the covariance matrix, and bootstrapping can also decrease the computational cost of posterior sampling. Mean-field approaches can reduce time complexity of sampling parameters from the posterior from quadratic $O(n^2)$ to linear $O(n)$ in the number of parameters $n$.\footnote{We describe space complexity in terms of the number of parameters $n$, so that we do not make assumptions about the underlying model.} However, assuming independence among parameters has been observed to result in worse performance in some settings \citep{RiquelmeTuSn18}. Low-rank approximations of the covariance matrix allow for sampling parameters in $O((n+1)\rho)$, where $\rho$ is the rank of the approximate covariance, but such methods have a space complexity of $O(\rho n)$ since they require storing $\rho$ copies of the parameters \cite{zhang2018noisy,maddox2019swag}. Bootstrapping also requires storing multiple copies of the parameters, so the space is $O(bn)$ where $b$ is the number of bootstrap replicates. However, bootstrapping simply requires a multinomial draw to select one set of bootstrapped parameters. All these methods require a forward pass using the sampled parameters, and the time complexity is the sum of the time complexities of sampling parameters and the forward pass.
\begin{table}[t!]
\centering
\caption{\label{table:complexity} Decision-making time complexity and space complexity for each method . For methods relying on fully-connected neural networks, the time complexity of a forward pass to the last hidden layer is
  $C_{\mbox{last-layer}} = dh_1 + \sum_{m=1}^{M-1} h_mh_{m+1}$, where
  $d$ is the dimension of the context and $h_m$ is the number of units in
  hidden layer $m$. For \textsc{Bootstrap-NN-TS}, $B$ denotes the number of bootstrap replicates.}
\begin{small}
\begin{sc}
\begin{tabular}{lcc}
\toprule
Method & Time Complexity & Space Complexity\\
\midrule
NeuralGreedy & $O(C_{\mbox{last-layer}}) + O(kh_M)$ & $O(C_{\mbox{last-layer}}) + O(kh_M)$\\
Linear-TS & $O(kd^2)$ & $O(kd^2)$\\
NeuralLinear-TS & $O(C_{\mbox{last-layer}}) + O\big(kh_M^2)$ & $O(C_{\mbox{last-layer}}) + O\big(kh_M^2)$\\
Bootstrap-NN-TS & $O(C_{\mbox{last-layer}}) + O(kh_M)$ & $O(C_{\mbox{last-layer}}\cdot B) + O(kh_MB))$\\
IL & $O(C_{\mbox{last-layer}}) + O(kh_M)$\\
\bottomrule
\end{tabular}
\end{sc}
\end{small}
\end{table}